\documentclass[accepted]{uai2022} 

\usepackage[american]{babel}

\usepackage{natbib} 
\usepackage{microtype}
\usepackage{graphicx}
\usepackage{subfigure}
\usepackage{booktabs} 
\usepackage{enumitem}
\usepackage{multirow}

\usepackage{amsmath}
\usepackage{mathtools}
\usepackage{amsthm}
\usepackage{amssymb}
\usepackage{mathtools} 
\usepackage{booktabs} 
\usepackage{tikz} 
\usepackage{algorithm}
\usepackage{algorithmic}

\newcommand{\mat}[1]{\mathbf{#1}}
\newcommand{\vect}[1]{\mathbf{#1}}
\newcommand{\norm}[1]{\left\|#1\right\|}
\newcommand{\inner}[1]{\left\langle#1\right\rangle}
\newcommand{\abs}[1]{\left|#1\right|}

\newcommand{\params}{\vect{\theta}}

\newcommand{\relu}[1]{\sigma\left(#1\right)}

\newcommand{\lambdamin}{\lambda_{\min}\left(\mat{K}^{(H)}\right)}

\newtheorem{thm}{Theorem}[section]
\newtheorem{lem}{Lemma}[section]

\newtheorem{condition}{Condition}[section]

\title{ResIST: Layer-Wise Decomposition of ResNets for Distributed Training}

\author[1]{Chen Dun}
\author[1]{Cameron R. Wolfe}
\author[1]{Christopher M. Jermaine}
\author[1]{Anastasios Kyrillidis}

\affil[1]{%
    Computer Science Dept.\\
    Rice University\\
    Houston, Texas, USA
}

\begin{document}
\maketitle

\begin{abstract}
We propose {\rm \texttt{ResIST}}, a novel distributed training protocol for Residual Networks (ResNets).
{\rm \texttt{ResIST}} randomly decomposes a global ResNet into several shallow sub-ResNets that are trained independently in a distributed manner for several local iterations, before having their updates synchronized and aggregated into the global model.
In the next round, new sub-ResNets are randomly generated and the process repeats until convergence.
By construction, per iteration, {\rm \texttt{ResIST}} communicates only a small portion of network parameters to each machine and never uses the full model during training.
Thus, {\rm \texttt{ResIST}} reduces the per-iteration communication, memory, and time requirements of ResNet training to only a fraction of the requirements of full-model training. 
In comparison to common protocols, like data-parallel training and data-parallel training with local SGD, {\rm \texttt{ResIST}} yields a decrease in communication and compute requirements, while being competitive with respect to model performance.
\end{abstract}

\begin{figure*}[ht]
  \centering
\includegraphics[width=\textwidth]{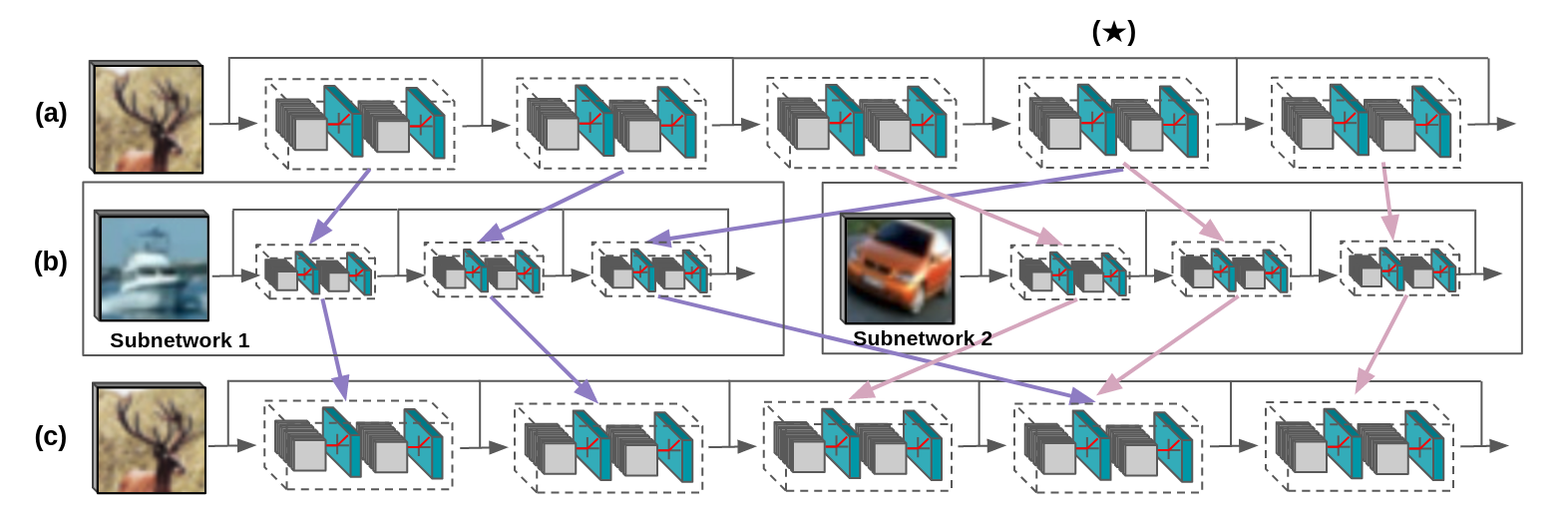}
\vspace{-0.3cm}
\caption{The \texttt{ResIST} model: \textbf{Row $\mathbf{(a)}$} represents the original global ResNet. 
\textbf{Row $\mathbf{(b)}$} shows the creation of two sub-ResNets. Observe that subnetwork 1 contains the residual blocks \#1, \#2 and \#4, while subnetwork 2 contains the residual blocks \#3, \#4 and \#5. 
\textbf{Row $\mathbf{(c)}$} shows the reassembly of the global ResNet, after locally training subnetworks 1 and 2 for some number of local SGD iterations; residual blocks that are common across subnetworks (e.g., residual block \#4, marked with a $\star$) are aggregated appropriately during the reassembly.}
\label{resist}
\vspace{-0.1cm}
\end{figure*}




\section{Introduction}
\textbf{Background.}
The field of Computer Vision (CV) has seen a revolution, beginning with the introduction of AlexNet during the ILSVRC2012 competition. 
Following this initial application of deep convolutional neural networks (CNNs), 
the introduction of the residual connection (ResNets) allowed scaling to massive depths without being crippled by issues of unstable gradients during training \citep{resnet}. 
The capabilities of ResNets have been further expanded in recent years, but the basic ResNet architecture has remained widely-used. 
While ResNets have become a standard building block for the advancement of CV research, 
the computational requirements for training them are significant.
For example, training a ResNet50 on ImageNet with a single NVIDIA M40 GPU takes 14 days \citep{you2018imagenet}.

Distributed training with multiple GPUs is commonly adopted to speed up the training process for ResNets. 
Yet, such acceleration is achieved at the cost of a remarkably large number of GPUs (e.g 256 NVIDIA Tesla
P100 GPU in \citep{1hrimagenet}).
Additionally, frequent synchronization and high communication costs create bottlenecks that hinder such methods from achieving speedups with respect to the number of available GPUs~\citep{distrib-benchmark}.
Asynchronous approaches avoid the cost of synchronization, but stale updates complicate their optimization process \citep{asynch-summary}.
Other methods, such as data-parallel training with local SGD \citep{localsgdconverge, use_local_sgd, parallel_sgd, fed_avg}, reduce the frequency of synchronization.
Similarly, model-parallel training has gained in popularity by decreasing the cost of local training between synchronization rounds \citep{parallelism_survey, lamp, nonlinear_multigrid_layer_parallel, layer_parallel_resnet, xpipe, multi_gpu_model_parallel}.



\textbf{This paper.}
We focus on efficient distributed training of CNNs with residual skip connections.
Our proposed methodology accelerates synchronous, distributed training by leveraging ResNet robustness to layer removal \citep{stochdepth}.
In particular, a group of high-performing subnetworks (sub-ResNets) is created by partitioning the layers of a shared ResNet model to create multiple, shallower sub-ResNets.
These sub-ResNets are then trained independently (in parallel) for several iterations before aggregating their updates into the global model and beginning the next iteration.
Through the local, independent training of shallow sub-ResNets, this methodology both limits synchronization and communicates fewer parameters per synchronization cycle, thus drastically reducing communication overhead.
We name this scheme \textit{ResNet Independent Subnetwork Training} (\texttt{ResIST}).
The contributions of this work are: \vspace{-0.3cm}
\begin{itemize}[leftmargin=*]
    \item We propose a distributed training scheme for ResNets, dubbed \texttt{ResIST}, that partitions the layers of a global model to multiple, shallow sub-ResNets, which are then trained independently between synchronization rounds. \vspace{-0.1cm}
    \item We provide theory that \texttt{ResIST} (based on simple ResNet architectures) converges linearly, up to an error neighborhood, using distributed gradient descent with local iterations. We show that the behavior of \texttt{ResIST} is controlled by the overparameterization parameter $m$, as well as the number of workers $S$ in the distributed setting, the number of local iterations, as well as the depth $H$ of the ResNet architecture. Such findings reflect practical observations that are made in the experimental section. \vspace{-0.1cm}
    \item We perform extensive ablation experiments to motivate the design choices for \texttt{ResIST}, indicating that optimal performance is achieved by $i)$ using pre-activation ResNets, $ii)$ scaling intermediate activations of the global network at inference time, $iii)$ sharing layers between sub-ResNets that are sensitive to pruning, and $iv)$ imposing a minimum depth on sub-ResNets during training. \vspace{-0.1cm}
    \item \texttt{ResIST} is shown to achieve high accuracy and time efficiency in all cases. We conduct experiments on several image classification and object detection datasets, including CIFAR10/100, ImageNet, and PascalVOC.  \vspace{-0.1cm}
    \item We utilize \texttt{ResIST} to train numerous different ResNet architectures (e.g., ResNet101, ResNet152, and ResNet200) and provide implementations for each in PyTorch \citep{pytorch}. \vspace{-0.1cm} 
\end{itemize}


\section{Sub-ResNet Training} \label{methods}
\texttt{ResIST} operates by partitioning the layers of a global ResNet to different, shallower sub-ResNets, training those independently, and intermittently aggregating their updates into the global model.
The high-level process followed by \texttt{ResIST} is depicted in Fig. \ref{resist} and outlined in more detail by Algorithm \ref{alg:resist}.
\emph{We note that a naive, uniform partitioning of blocks to each subnetwork, resembling a distributed implementation of \citep{stochdepth}, performs poorly (see Fig. \ref{resist_ablations}).}
To improve upon this procedure, extensive design choices, outlined in Section \ref{design_ablation} in the Appendix, are studied to motivate \texttt{ResIST}, leading to a final methodology that generalizes well across domains and datasets.

\vspace{-0.1cm}
\subsection{Model Architecture} \label{model_arch}
\vspace{-0.1cm}

\begin{figure}
\centering
\includegraphics[width=1\columnwidth]{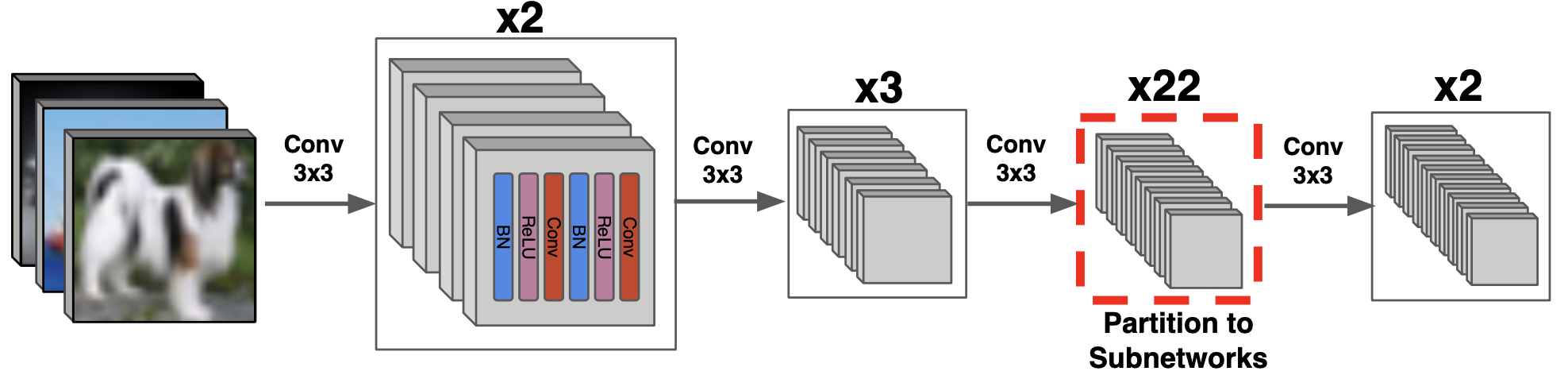}
\vspace{-0.2cm}
\caption{The ResNet101 model used in the majority of experiments. The figure identifies the convolutional blocks that are partitioned to subnetworks. The plot depicts the pre-activation ResNet setting, where we use BN, ReLU, and Conv layers twice in sequence. The network is comprised of four major ``sections'', each containing a certain number of convolutional blocks of equal channel dimension.}
\label{model_depict}
\vspace{-0.4cm}
\end{figure}

To achieve optimal performance with \texttt{ResIST}, the global model must be sufficiently deep.
Otherwise, sub-ResNets may become too shallow after partitioning, leading to poor performance.
For most experiments, a ResNet101 architecture is selected, which balances sufficient depth with reasonable computational complexity. 
Experiments with deeper architectures are provided in Section \ref{S:deep_arch} in the Appendix.

\texttt{ResIST} performs best with pre-activation ResNets \citep{preactres}.
Intuitively, applying batch normalization prior to the convolution ensures that the input distribution of remaining residual blocks will remain fixed, even when certain layers are removed from the architecture.
The Pre-activation ResNet101, which we utilize for the majority of experiments, is depicted in Fig. \ref{model_depict}.
This model, as well as deeper variants (e.g., ResNet152 and ResNet200), are readily available through deep learning packages like PyTorch \citep{pytorch} and Tensorflow \citep{tensorflow}.

\vspace{-0.1cm}
\subsection{Sub-ResNet Construction} \label{subnet_sec}
\vspace{-0.1cm}

Pruning literature has shown that strided-, initial-, and final-layers within CNNs are sensitive to pruning \citep{filterprune}.
Additionally, repeated blocks of identical convolutions (i.e., equal channel size and spatial resolution) are less sensitive to pruning \citep{filterprune}.
Drawing upon these results, \texttt{ResIST} only partitions blocks within the third section of the ResNet (see the highlighted section in Fig. \ref{model_depict}), while all other blocks are shared between sub-ResNets.
These blocks are chosen for partitioning because $i)$ they account for the majority of layers; $ii)$ they are not strided; $iii)$ they are located within the middle of the network (i.e., initial/final layers are excluded); and $iv)$ they reside within a long chain of identical convolutions.
By partitioning these blocks, \texttt{ResIST} allows sub-ResNets to be shallower than the global model, while maintaining high performance. 

The process of constructing sub-ResNets follows a simple procedure; see Figure \ref{resist}.
From row $(a)$ to $(b)$ within Figure \ref{resist}, indices of partitioned layers within the global model are randomly permuted and distributed to sub-ResNets in a round-robin fashion.
Each sub-ResNet receives an equal number of convolutional blocks (e.g., see row $(b)$).
In cases, residual blocks may be simultaneously partitioned to multiple sub-ResNets to ensure sufficient depth (e.g., see $(\star)$ in Figure \ref{resist}).
\texttt{ResIST} produces subnetworks with $\mathcal{O}(\frac{1}{S})$ of the global model depth, where $S$ is the number of independently-trained sub-ResNets.\footnote{A fixed number of blocks is excluded from partitioning (i.e., blocks not in the third section). As a result, this approximation of $\mathcal{O}(\frac{1}{S})$ becomes more accurate as the network becomes deeper (i.e., deeper ResNet variants only add blocks to the third section), as a larger ratio of total blocks are included in the partitioning process.}
To contrast this with existing non-distributed attempts, stochastic depth networks \citep{stochdepth} have an expected depth of 75\% of the global model.

The shallow sub-ResNets created by \texttt{ResIST} accelerate training and reduce communication in comparison to methods that communicate and train the full model.
Table \ref{comm_amounts} shows the comparison of local SGD to \texttt{ResIST} with respect to the amount of data communicated during each synchronization round for different numbers of machines, highlighting the superior communication-efficiency of \texttt{ResIST}.

\begin{table}[!t]
\centering
\caption{Data communicated during each communication round (in GB) of both local SGD \citep{localsgdconverge} and \texttt{ResIST} across different numbers of machines with ResNet101.}
\vspace{0.2cm}
\begin{tabular}{cccccc}
\toprule
    Method & 2 Machine & 4 Machine & 8 Machine\\ \midrule
    Local SGD & 0.662 GB & 1.325 GB & 2.649 GB\\
    \texttt{ResIST} & \textbf{0.454} GB & \textbf{0.720} GB & \textbf{1.289} GB \\
 \bottomrule
\end{tabular}
\label{comm_amounts}
\end{table}



\subsection{Distributed Training}
The \texttt{ResIST} training procedure is outlined in Algorithm \ref{alg:resist}.
Sub-ResNet construction (i.e., \texttt{subResNets$(\cdot)$} in Algorithm \ref{alg:resist}) follows the procedure outlined in Sec. \ref{subnet_sec}.
After constructing the sub-ResNets, they are trained independently in a distributed manner for $\ell$ iterations. 
Following independent training, the updates from each sub-ResNet are aggregated into the global model.
Aggregation (i.e., $\text{\texttt{aggregate}}(\cdot)$ in Algorithm \ref{alg:resist}) sets each global network parameter to its average value across the sub-ResNets to which it was partitioned.
If a parameter is only partitioned to a single sub-ResNet, aggregation simplifies to copying the parameter into the global model.
After aggregation, the global model is re-partitioned randomly to create a new group of sub-ResNets, and this entire process is repeated.

\begin{algorithm}[!htp]
\centering
\caption{\textsc{ResIST} Meta Algorithm}\label{alg:resist}
\begin{algorithmic}
    \STATE \textbf{Parameters}: $T$ synchronization iterations, $S$ sub-ResNets, $\ell$ local iterations, $\mathcal{W}$ ResNet weights. 
    \STATE $h(\mathcal{W})$ $\leftarrow$ randomly initialized ResNet.
    \FOR{$t = 0, \dots, T-1$}
    \STATE$\left\{h_s(W_s)\right\}_{s = 1}^S = \text{\texttt{subResNets}}(h(W), ~S)$.
    \STATE Distribute each $h_s(W_s)$ to a different worker.
    \FOR{$s = 1, \dots, S$}
    \STATE //~Train $h_s(W_s)$ for $\ell$ iterations using local SGD.
    \FOR{$l_t = 1, \dots, \ell$}
    \STATE $W_s=W_s-\eta \frac{\partial L}{W_s}$
    \ENDFOR
    \ENDFOR
    \STATE $h(\mathcal{W}) = \texttt{aggregate}\left(\left\{h_s(W_s)\right\}_{s = 1}^S\right)$.
\ENDFOR        
\end{algorithmic}
\end{algorithm}

\begin{figure}
\vspace{0.1cm}
\centering
\includegraphics[width=0.95\columnwidth]{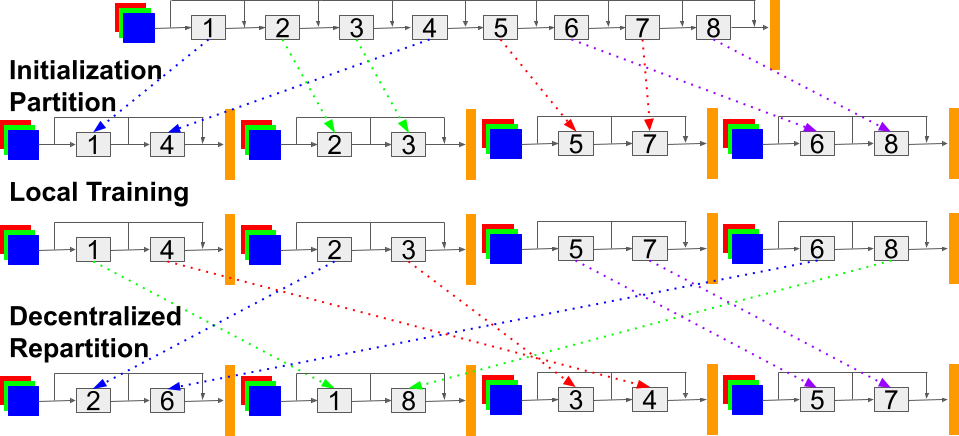}
\vspace{-0.2cm}
\caption{A depiction of the decentralized repartition procedure. This example partitions a ResNet with eight blocks into four different sub-ResNets. The ``blue-green-red'' squares dictate the data that lies per worker; the orange column dictates the last classification layer. As seen in the figure, each worker is responsible for only a fraction of parameters of the whole network. The whole ResNet is never fully stored, communicated or updated on a single worker.}
\label{decentralized}
\end{figure}

\begin{figure}[h]
    \vspace{-0.4cm}
    \centering
    \includegraphics[width=1\linewidth]{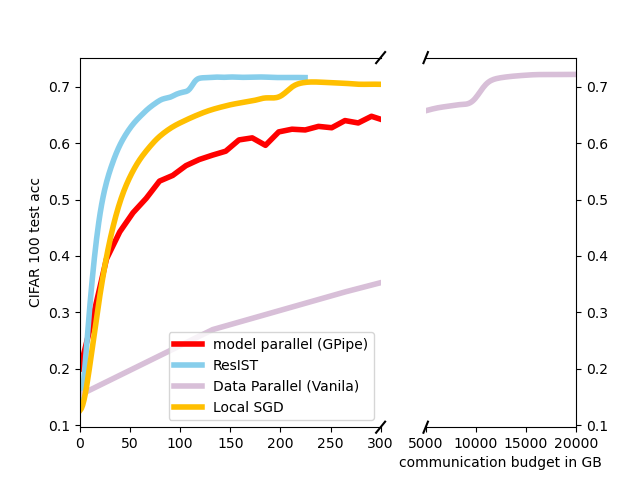}
    \vspace{-0.6cm}
    \caption{Communication efficiency of \texttt{ResIST} versus data parallelism (vanila), model parallelism (GPipe - \citep{huang2019gpipe}) and local SGD (LSGD) on CIFAR100. }
    \label{fig:baseline}
    \vspace{-0.3cm}
\end{figure}

\subsection{Baseline Choice}
Common baselines for distributed training are generally split into data- and model-parallelism protocols. 
Focusing on the former, the communication efficiency of \texttt{ResIST} significantly surpasses data-parallelism.
In particular, data parallel methods need to synchronize the whole model at every training iteration, while \texttt{ResIST} only needs to communicate the weights of sub-ResNets among the workers. 

Typically, model parallel techniques split the model into modules (such as layers) and distribute these modules to each worker.
At every training iteration, input data is first passed to the worker containing the network's beginning module (e.g., the first layer).
Then, at each module, the worker $i)$ performs a forward pass of its module and $ii)$ sends the resulting output activation to the worker containing the next module.
After the last module is activated, the final loss and gradient is calculated before the backward pass is performed, where each worker receives gradient information needed for updating module weights.

Model-parallelism often suffers from higher communication frequency and volume, in comparison to data parallel methods, due to the significant cost of transmitting network activation maps between workers. 
E.g., model parallel training of ResNets requires transmission of the full batch activation map between layers, which is more cumbersome than simply communicating network parameters.
\texttt{ResIST} is more communication efficient compared to common model parallel methods (e.g., GPipe \citep{huang2019gpipe}).

Within this work, we adopt local SGD \citep{use_local_sgd}---a strong variant of data parallel training---as our baseline.
Similar to \texttt{ResIST}, local SGD performs local training iterations on each worker between synchronizations, thus largely decreasing communication frequency and volume. 
To justify this selection, we perform a baseline comparison, which is displayed in Figure \ref{fig:baseline} and further detailed in Sections \ref{S:Implementation} and \ref{S:exp_det}. 
As shown in Figure \ref{fig:baseline}, \texttt{ResIST} is significantly more communication efficient in comparison to data-parallelism (vanilla) and model-parallelism (GPipe), thus making local SGD a more appropriate baseline.



\subsection{Implementation Details}
\label{S:Implementation}

\texttt{ResIST} is implemented in PyTorch \citep{pytorch}, using the NCCL communication package. 
We use basic \texttt{broadcast} and \texttt{reduce} operations for communicating blocks in the third section and \texttt{all reduce} for blocks in other sections. 
We adopt the same communication procedure for the local SGD baseline to ensure fair comparison.  
\emph{The implementation of \texttt{ResIST} is decentralized, meaning that it does not assume a single, central parameter server.}

As shown in Figure \ref{decentralized}, during the synchronization and repartition step following local training, each sub-ResNet will directly send each of its locally-updated blocks to the designated new sub-ResNet. 
Each worker will only need sufficient memory to store a single sub-ResNet, thus limiting the memory requirements.
Such a decentralized implementation allows parallel communication between sub-ResNets, which leads to further speedups by preventing any single machine from causing slow-downs due to communication bottlenecks. 
The implementation is easily scalable to eight or more machines, either on nodes with multiple GPUs or across distributed nodes with dedicated GPUs. 

\texttt{ResIST} reduces the number of bits communicated at each synchronization round and accelerates local training with the use of shallow sub-ResNets.
The authors are well-aware of many highly-optimized versions of data-parallel and synchronous training methodologies \citep{pytorch, tensorflow, sergeev2018horovod}. 
\texttt{ResIST} is fully compatible with these frameworks and can be further accelerated by leveraging highly-optimized distributed communication protocols at the systems level, which we leave as future work. 
Further, the authors are well-aware of advanced recent decentralized distributed computing techniques as in \citep{koloskova2020unified, nedic2009distributed, assran2020asynchronous, koloskova2019decentralized}; 
our aim is to show the benefits of our approach even on simpler distributed frameworks, and we leave the extension of \texttt{ResIST} to such more advanced protocols as future work.

\subsection{Supplemental Techniques}
\label{S:supp_tech}

\textbf{Scaling Activations.}
Similar to \citep{stochdepth}, activations must be scaled appropriately to account for the full depth of the resulting network at test time.
To handle this, the output of residual blocks in the third section of the network (see Figure \ref{model_depict}) is scaled by $1/S$, where $S$ is the number of sub-ResNets.
Such scaling allows the global model to perform well, despite using all layers at test time. 

\noindent
\textbf{Subnetwork Depth.}
Within \texttt{ResIST}, sub-ResNets may become too shallow as the number of sub-ResNets increases. 
To solve this issue, \texttt{ResIST} enforces a minimum depth requirement, which is satisfied by sharing certain blocks between multiple sub-ResNets.
Through experimental analysis, a minimum of five blocks partitioned to each sub-ResNet was found to perform optimally.
Such a finding motivates our choice of the ResNet101 architecture, as ResNet50 contains only five blocks for partitioning.
\texttt{ResIST} is extensible to deeper architectures; see Section \ref{S:deep_arch} in the Appendix.

\noindent
\textbf{Tuning Local Iterations.}
We use a default value of $\ell=50$, as $\ell<50$ did not noticeably improve performance.
In some cases, the performance of \texttt{ResIST} can be improved by tuning $\ell$ (see Figure \ref{fig:local_iter}).
The optimal $\ell$ setting in \texttt{ResIST} is further explored in Section \ref{S:local_iter} in the Appendix.

\noindent
\textbf{Local SGD Warm-up Phase.}
Directly applying \texttt{ResIST} may harm performance on some large-scale datasets (e.g., ImageNet).
To resolve this, we perform a few epochs with data parallel local SGD before training the model with \texttt{ResIST}.\footnote{Activations of blocks within $3^{\text{rd}}$ section are still scaled during local SGD pre-training to maintain consistency with \texttt{ResIST}.}
By simply pre-training a model for a few epochs with local SGD, the remainder of training is completed using \texttt{ResIST} without a significant performance decrease.

\vspace{-0.1cm}
\section{Theoretical Result}
\vspace{-0.1cm}
We provide proof that the gradient descent direction of combined updates from all sub-ResNets, during distributed local training, is close to the hypothetical gradient descent direction of the whole model as if trained centrally.     

\begin{thm}[Convergence Rate of Gradient Descent for \texttt{ResIST}]\label{thm:resist_gd}
Assume there are $S$ workers, $\ell$ local and $T$ global steps. Assume the depth of the whole ResNet is $H$. Assume for all data indices $i \in [n]$, the data input satisfies $\norm{\vect{x}_i}_2 = 1$, the data output satisfies $\abs{y_i} = O(1)$, and the number of hidden nodes per layer satisfies $m = $ 
\begin{align}
\Omega\bigg(\max\bigg\{\tfrac{n^4 }{\lambda_{\min}^4\left(\mat{K}^{(H)}\right)H^6},\tfrac{n^2 }{\lambda_{\min}^2(\mat{K}^{(H)})H^2},\tfrac{n}{\delta}, \tfrac{n^2\log\left(\tfrac{Hn}{\delta}\right)}{\lambda_{\min}^2\left(\mat{K}^{(H)}\right)} \bigg\}\bigg).\nonumber
\end{align}
Set the step size $\eta = O\bigg(\tfrac{\lambdamin H^2 }{n^2 \ell^2 S}\bigg)$ in gradient descent in local training iteration, and follow the procedure as in Algorithm 1. 
Let the squared-norm loss be $L(\theta(t)) := \tfrac{1}{2} \|\vect{y} - f(\theta(t))\|_2^2$, per $t$ global synchronization round, $t=1,2,\ldots T$; here, $\vect{y}$ corresponds to the data ``labels'', and $\theta(t)$ and $f(\theta(t))$ represent the parameters and the output of the whole ResNet, respectively, after $t$-global rounds of \texttt{ResIST}.
Here, $\theta$ includes weights $\mat{W}^{(h)}$ at depth $h$ and the last layer's weights $\mat{a}$.
Then, with probability at least $1-\delta$ over the random initialization, we have:
\begin{align*}
L(\params(t))\le \left(1-\tfrac{\eta \ell \lambdamin}{2}\right)^{t} \cdot L(\params(0)).
\end{align*}
\vspace{-0.8cm}
\end{thm}

First, some definitions; more details in the Appendix.
Similar to \citep{du2019gradient}, $\mat{K}^{(H)} \in \mathbb{R}^{n \times n}$ 
is a fixed matrix that depends on the input data, neural network architecture and the activation but does not depend on neural network parameters.
Next, we present our method of proving this global result on \texttt{ResIST}. 
Our proof technique is inspired by \citep{du2019gradient}: 
Let the prediction of the network at some $k$-th iteration be $\vect{u}(k)=f(\theta(k))$.\footnote{We use $k$ to abstract the notion of an iteration in \citep{du2019gradient}; in our case, a different analysis includes two different iteration indices, $\ell$ and $t$.}
We formulate the training dynamics as: \vspace{-0.2cm}
\begin{align*}
\vect{y}-\vect{u}(k+1) = (\vect{I}-\eta \mat{G}(k))(\vect{y}-\vect{u}(k)), 
\end{align*}
where $\mat{G}_{ij}(k) = \inner{\frac{\partial u_i(k)}{\partial \params(k)},\frac{\partial u_j(k)}{\partial \params(k)}} = $ \vspace{-0.2cm}
\begin{align*}
&\sum_{h=1}^{H}\inner{\tfrac{\partial u_i(k)}{\partial \mat{W}^{(h)}(k)},\tfrac{\partial u_j(k)}{\partial \mat{W}^{(h)}(k)}} + \inner{\tfrac{\partial u_i(k)}{\partial \vect{a}(k)},\tfrac{\partial u_j(k)}{\partial \vect{a}(k)}} \\
\triangleq & \sum_{h=1}^{H+1}\mat{G}^{(h)}_{ij}(k).
\end{align*}
The proof in \citep{du2019gradient} obeys the following ideas: when the width $m$ of deep ResNet is sufficiently large, $\mat{G}^{(H)}(k)$ will be very close to $\mat{G}^{(H)}(0)$, and all of $\mat{G}^{(H)}(k)$'s will be close to the fixed population gram matrix $\mat{K}^{(H)}$. The exact definition of $\mat{K}^{(H)}$ for ResNet can be found in Section 6 of \cite{du2019gradient}. Further, $\lambda_{\min} (\mat{G}^{(H)}(0))$ is larger than 0. Thus, by standard matrix perturbation analysis, it is shown that $\lambda_{\min}(\mat{G}^{(H)}(0))$ is also strictly positive, which will result in linear convergence of deep ResNet.

Here, we further generalize such technique to distributed \texttt{ResIST} with layer dropout.
The novelty of our proof is that we only conduct gradient descent on sub-ResNets assigned to each local worker. 
\emph{There is no training iteration with the whole model: this includes the generation of random masks that ``champion'' parts of the whole ResNet model per worker.} Handling such constructions is the gist of this proof:
We carefully analyze the convergence of each subnetwork during local training iterations $\ell$, and prove the global convergence of the combined whole model throughout synchronization rounds $t$. 
The full proof is provided in Section \ref{sec:proof} in the Appendix.

\vspace{-0.1cm}
\section{Related Work} \label{related_work}
\vspace{-0.1cm}

Following ResNet, 
most novel architectures continued to leverage residual connections, 
which became standard in most architectures. 
The ResNet architecture has been further modified. 
\emph{This work focuses on the pre-activation ResNet variant \citep{preactres}, as it achieves high performance and is well-suited to layer-wise decomposition.}

The focus of this study is on synchronous methods of distributed optimization, such as data parallel training, parallel SGD \citep{parallelsgd}, or local SGD \citep{localsgdconverge}.
Our methodology is also a variant of model-parallel training \citep{parallelism_survey, lamp,  nonlinear_multigrid_layer_parallel, layer_parallel_resnet, xpipe, multi_gpu_model_parallel}.
Many studies have explored possible techniques of synchronous, distributed optimization, yielding a wide number of viable variants \citep{use_local_sgd}. 

To reduce communication costs in the distributed setting, both quantization \citep{commeff-sgd, comm-comp} and sparsification \citep{sparse-comm, linear-speed-quant, grad-sparse} methods have been explored.
Similarly, other studies have achieved speedups through the use of low-precision arithmetic during training \citep{4minimagenet}.
However, \emph{this line of work is orthogonal to our proposal and can be easily combined with the provided methodology; see Section \ref{S:quant} in the Appendix.}  


Large batch training is used to amortize communication and increase throughput for distributed training \citep{1hrimagenet, 76minbert}.
The properties of large batch training have since been studied extensively \citep{15minimagenet, 32Kbatch, you2018imagenet}. 
Large batches alter training dynamics, warranting the use of complex heuristics to maintain comparable performance \citep{32Kbatch}.
Here, \emph{we do not focus on the extension of \texttt{ResIST} to the large-batch training domain.
Rather, we consider this as future work. }
ResNet robustness to layer removal was explored in \citep{stochdepth}, while \citep{shallowensemble} showed that ensembles of shallow ResNets can yield high performance.  
\citep{stochdepth} uses shallow networks during training and scales activations so that all layers may be used for inference. 
However, our approach is distinct in numerous ways.
Primarily, \emph{our method partitions blocks in a stochastic, round-robin fashion, which explicitly prevents the exclusion of layers from training rounds and yields reduced subnetwork depth compared to \citep{stochdepth}.}
Inspired by \citep{filterprune}, we also selectively partition residual blocks that are least sensitive to pruning, allowing other layers (i.e., ~30\% of total layers) to be shared between subnetworks.
Unlike \citep{stochdepth}, we avoid partitioning strided layers, which are sensitive to pruning \citep{filterprune}.
Furthermore, our methodology, instead of proposing a form of regularization, focuses on utilizing independent training of shallow sub-ResNets for efficient, distributed training. 

Our approach also relates to neural ODE literature.
This research connects ResNets as a discrete approximation to a continuous transformation from input to output \citep{beyondfinite}. 
The neural ODE perspective has been studied both empirically \citep{neuralode, augmentedneuralode, beyondfinite} and theoretically \citep{mfresnet, deeplimit}. 
\emph{This provides justification to our approach, as removing ResNet layers can be viewed as approximating the same transformation with a coarser discretization.}

\section{Experimental Details}
\label{S:exp_det}
Hyperparameters are tuned using a holdout validation set and results are obtained using optimal hyperparameters from the validation set. 
All experiments are repeated for three trials, and the average performance is presented. 
We adopt local SGD as our baseline for synchronous, distributed training methods. 
\emph{In all cases, \texttt{ResIST} achieves comparable performance to local SGD, while lowering the total wall-clock time of training.}
We use AWS p3.8xlarge instances for experiments with two or four machines\footnote{In Section \ref{S:deep_arch}-Appendix and for the Pacal VOC experiment with two machines, we use a cluster with eight V100 GPUs.} and p3.16xlarge instances for experiments with eight machines. 
We use each GPU as a single worker that hosts a different sub-ResNet.

\noindent
\textbf{Small-Scale Image Classification.}
Models are trained with \texttt{ResIST} on CIFAR10 and CIFAR100 for image classification.
We adopt standard data augmentation techniques during training and testing \citep{resnet}.
We adopt a batch size of 128 for each worker.
Training is conducted for 80 epochs for experiments with two machines and 160 epochs for experiments with four or eight machines.
The recorded performance reflects the best test accuracy achieved throughout training, averaged across three trials.
The total wall-clock training time is also reported for each experiment.

\noindent
\textbf{ImageNet Classification.}
Models are trained with \texttt{ResIST} on the 1,000-class ILSVRC2012 image classification dataset \citep{imagenet}.
We adopt standard data augmentation techniques during training and testing, and use a batch size of 256 for each worker  \citep{resnet}.
Training is conducted for 90 epochs.
We initialize the learning rate to 0.1 and decrease it $10\times$ at epochs 30 and 60.
For all experiments, we set $\ell=15$, adopt a minimum depth of 10 blocks for each sub-ResNet, and warm-up pre-training using local SGD. 
For both \texttt{ResIST} and baseline experiments, we utilize momentum restarts and aggregate batch statistics every 1300 synchronization rounds.

\noindent
\textbf{Object Detection.}
\texttt{ResIST} is tested in the object detection domain on the Pascal VOC dataset \citep{pascalvoc}. 
Our model, inspired by the Yolo-v2 object detection model \citep{yolov2}, consists of a ResNet101 backbone followed by a detection layer (i.e., a $1 \times 1$ convolution that outputs anchor box predictions).
The ResNet backbone of this model is similar to the classification model described in Sec. \ref{model_arch}, but without the pre-activation structure.
The model is trained for 100 epochs with an image dimension of $448\times448$ and batch size of 10.
No data augmentation techniques are used.
The learning rate is increased from $10^{-5}$ to $10^{-4}$ over the first 30 epochs, and decreased by $10\times$ at epochs 60 and 90.
Both Pascal VOC 2007 and 2012 training sets are used during training, and performance is evaluated on the Pascal VOC 2007 test set.
We report the wall-clock training time and the best loss achieved on the test set throughout training.
Experiments are conducted on two and four machines using both local SGD and \texttt{ResIST}.

\begin{table}[!htp]
\centering
\caption{Test accuracy of baseline LocalSGD versus \texttt{ResIST} on small-scale image classification datasets.} 
\setlength{\tabcolsep}{.5\tabcolsep}
\begin{small}
\begin{tabular}{cccc}
\toprule
    & \# Machines & CIFAR10 & CIFAR100 \\ \midrule
    Local SGD & 2 & 92.36\% $\pm$ 0.01 & 70.67\% $\pm$ 0.03  \\
    & 4 & 92.90\% $\pm$ 0.06 & 71.51\% $\pm$ 0.04  \\
    & 8 & 92.00\% $\pm$ 0.07 & 69.64\% $\pm$ 0.05  \\
    \midrule
    \texttt{ResIST} & 2 & 91.95\% $\pm$ 0.32 & 70.06\% $\pm$ 0.51  \\
    & 4 & 92.35\% $\pm$ 0.22 & 71.30\% $\pm$ 0.20  \\
    & 8 & 91.45\% $\pm$ 0.30 & 70.26\% $\pm$ 0.21  \\
 \bottomrule
\end{tabular}
\end{small}
\label{cifar10_results}
\vspace{-0.3cm}
\end{table}

\vspace{-0.1cm}
\section{Results}
\vspace{-0.1cm}
\subsection{Small-Scale Image Classification}

\begin{table*}[!htp]
\centering
\caption{Performance of baseline models and models trained with \texttt{ResIST} on 1K Imagenet \citep{recht2019imagenet}. MF stands for test set ``MatchedFrequency'' and was sampled to match the MTurk selection frequency distribution of the original ImageNet validation set; T-0.7 stands for test set ``Threshold0.7'' and was built by sampling ten images for each class among the candidates with selection frequency at least 0.7; TI stands for test set ``TopImages'' and contains the ten images with highest selection frequency for each class.} \vspace{0.1cm}
\setlength{\tabcolsep}{.5\tabcolsep}
\begin{small}
\begin{tabular}{cccccccccc}
\toprule
    & \multirow{2}{*}{\# Machines} & \multirow{2}{*}{Imagenet} &  \multicolumn{3}{c}{Imagenet V2 Test Set}  & \multirow{2}{*}{Training Time} & \multirow{2}{*}{Speedup} & \multirow{2}{*}{Communication} & \multirow{2}{*}{Cost Ratio}\\ 
    & & & MF & T-0.7 & TI & & & & \\ \midrule
    Local SGD & 2 & 73.32\% & 60.72\% & 69.47\% & 75.48\% & 48.61 hours & - & 7546.80 GB & -\\
    & 4 & 72.66\% & 59.88\% & 68.34\% & 74.27\% & 29.29 hours & - & 7546.80 GB & -  \\
    \midrule
    \texttt{ResIST} & 2 & 71.60\% & 58.92\% & 67.51\% & 73.56\% & 36.79 hours &  \textbf{1.32$\times$} & 5831.2 GB & \textbf{1.29$\times$}\\
    & 4 & 70.74\% &57.56\% & 66.46\% & 72.65\% & 22.37 hours &  \textbf{1.31$\times$} & 6007.6 GB & \textbf{1.26$\times$}\\
 \bottomrule
\end{tabular}
\end{small}
\label{imagenet_results}
\end{table*}


\noindent
\textbf{Accuracy.}
The test accuracy on small-scale image classification datasets is listed in Table \ref{cifar10_results}.
\emph{\texttt{ResIST} achieves comparable test accuracy in all cases where the same number of machines are used.}
\texttt{ResIST} outperforms localSGD on CIFAR100 experiments with eight machines.
The performance of \texttt{ResIST} and local SGD are strikingly similar in terms of test accuracy.
In fact, the performance gap between the two method does not exceed 1\% in any experimental setting.
Furthermore, \texttt{ResIST} performance remains stable as the number of sub-ResNets increases, allowing greater acceleration to be achieved without degraded performance (e.g., see CIFAR100 results in Table \ref{cifar10_results}).
Generally, using four sub-ResNets yields the best performance with \texttt{ResIST}.


\noindent
\textbf{Efficiency.}
In addition to achieving comparable test accuracy, \texttt{ResIST} significantly accelerates training.
This acceleration is due to $i)$ fewer parameters being communicated between machines and $ii)$ locally-trained sub-ResNets being shallower than the global model.
Wall-clock training times for four and eight machine experiments are presented in Tables \ref{small_dataset_results}. 
\texttt{ResIST} provides $3.58$ to $3.81\times$ speedup in comparison to local SGD.
For eight machine experiments, a significant speedup over four machine experiments is not observed due to the minimum depth requirement and a reduction in the number of local iterations to improve training stability.
We conjecture that for cases with higher communication cost at each synchronization and a similar number of synchronizations, eight worker \texttt{ResIST} could lead to more significant speedups in comparison to the four worker case.

\begin{table}[!htp]
\centering
\vspace{-0.4cm}
\caption{Training time in seconds of baseline models and models trained with \texttt{ResIST} on small-scale image classification datasets.} \vspace{0.1cm}
\setlength{\tabcolsep}{.5\tabcolsep}
\begin{small}
\begin{tabular}{ccccccc}
\toprule
    & \# Machines & Dataset & Total Time & Speedup \\ \midrule
    Local SGD & 4 & C10 & 5486 $\pm$ 7.05 & -\\
    & & C100 & 5528 $\pm$ 65.90 & - \\
     & 8 & C10 & 10072 $\pm$ 5.12 & -\\
      &  & C100 & 10058 $\pm$ 8.71 & -\\
    \midrule
        \texttt{ResIST} & 4 & C10 & 1532 $\pm$ 0.83 & \textbf{3.60$\times$ }\\
    & & C100 &  1545 $\pm$ 1.27 & \textbf{3.58$\times$} \\
     & 8 & C10 & 2671 $\pm$ 3.25 & \textbf{3.77$\times$}\\
      &  & C100 & 2639 $\pm$ 3.89 & \textbf{3.81$\times$}\\
 \bottomrule
\end{tabular}
\end{small}
\label{small_dataset_results}
\end{table}

\begin{figure}[h]
    \centering
    \includegraphics[width=1\linewidth]{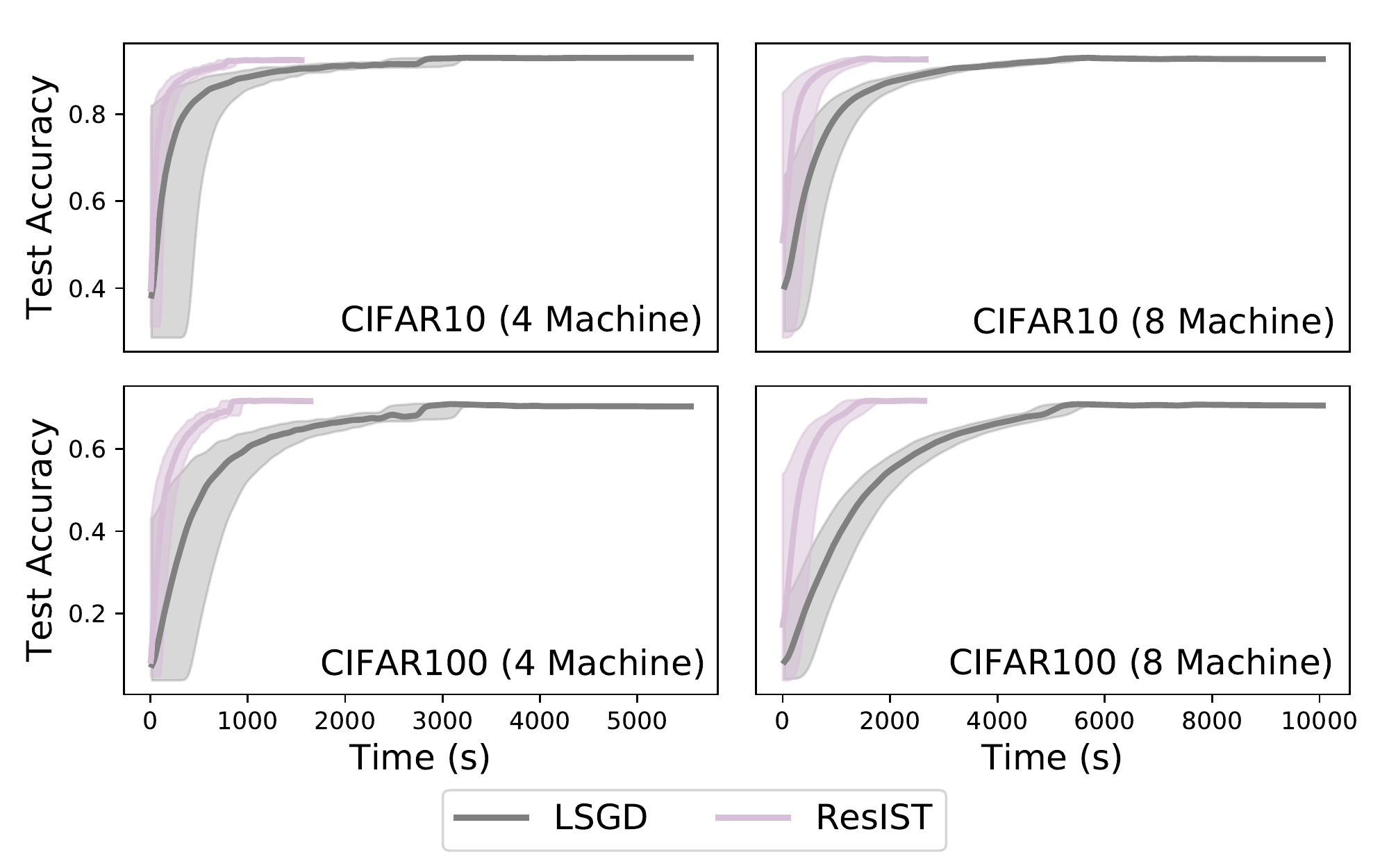}
    \vspace{-0.6cm}
    \caption{
    Both methodologies complete 160 epochs of training. Accuracy values are smoothed using a 1-D gaussian filter, and shaded regions represent deviations in accuracy.}
    \label{fig:cifar10_times}
\end{figure}

A visualization of the speedup provided by \texttt{ResIST} on the CIFAR10 and CIFAR100 datasets is illustrated in Fig. \ref{fig:cifar10_times}.
Models trained with \texttt{ResIST} match the final accuracy of those trained with local SGD.
Furthermore, increasing the number of sub-ResNets yields an improved speedup for \texttt{ResIST} in comparison to localSGD. 
It is clear that the communication-efficiency of \texttt{ResIST} allows the benefit of more devices to be better realized in the distributed setting. 

\subsection{Large-Scale Image Classification}

\noindent
\textbf{Accuracy.} The test accuracy of models trained with both \texttt{ResIST} and local SGD for different numbers of machines on the ImageNet dataset is listed in Table \ref{imagenet_results}.
As can be seen, \emph{\texttt{ResIST} achieves comparable test accuracy ($<2\%$ difference) to local SGD in all cases.}
Additionally, as shown in \citep{recht2019imagenet}, many current image classification models overfit to the ImageNet test set and cannot generalize well to new data. 
Thus, models trained with both local SGD and \texttt{ResIST} are also evaluated on three different Imagenet V2 testing sets \citep{recht2019imagenet}.
As shown in Table \ref{imagenet_results}, \texttt{ResIST} consistently achieves comparable test accuracy in comparison to local SGD on these supplemental test sets. 


\noindent
\textbf{Efficiency.} 
As shown in Tables \ref{imagenet_results} and \ref{reach_imagenet_speed}, \texttt{ResIST} significantly accelerates the ImageNet training process.
However, due to the use of fewer local iterations and the local SGD warm-up phase, the speedup provided by \texttt{ResIST} is smaller relative to experiments on small-scale datasets.
In Table \ref{imagenet_results}, \texttt{ResIST} can reduce the total communication volume during training, which is an important feature in the implementation of distributed systems with high computational costs.

\begin{table}[!htp]
\centering
\begin{small}
\caption{Total training time on Imagenet (in hours) of models trained with both local SGD and \texttt{ResIST} using two and four machines to reach a fixed test accuracy.} \vspace{0.1cm}
\setlength{\tabcolsep}{.5\tabcolsep}
\begin{tabular}{ccccc}
\toprule
    \# Machines & Target Accuracy & Local SGD & \texttt{ResIST} & Speedup \\ \midrule
    2  & 71.00 & 33.26 & 26.63 & \textbf{1.25$\times$ }\\
    4  & 70.70 & 18.50 & 18.12 & \textbf{1.02$\times$ }\\

 \bottomrule
\end{tabular}
\label{reach_imagenet_speed}
\end{small}
\vspace{-0.3cm}
\end{table}

\subsection{Object Detection}
\textbf{Loss.} The test loss of models trained with both \texttt{ResIST} and local SGD for different numbers of machines on the Pascal VOC object detection dataset is listed in Table \ref{obj_detection}.
Notably, \texttt{ResIST} achieves a lower test loss in comparison to local SGD for the experiment with two machines.
Although the test loss achieved by \texttt{ResIST} is slightly worse than local SGD in the four machine case, the performance is comparable.
Namely, the difference in test loss achieved by local SGD and \texttt{ResIST} never exceeds a value of one.

\noindent
\textbf{Efficiency.} In addition to achieving comparable or improved test loss in comparison to local SGD, \texttt{ResIST} also provides a significant training acceleration on the PascalVOC dataset.
In particular, models trained with \texttt{ResIST} achieve up to a $1.64\times$ acceleration in comparison to object detection models trained with localSGD. 

\begin{table}[!htp]
\centering
\vspace{-0cm}
\begin{small}
\caption{Test loss and total training time in seconds on Pascal VOC for models trained with both local SGD and \texttt{ResIST} using two and four machines. Training time in seconds.} \vspace{0.1cm}
\setlength{\tabcolsep}{.5\tabcolsep}
\begin{tabular}{ccccc}
\toprule
    & \# Machines & Test Loss & Train Time & Speedup \\ \midrule
    Local SGD & 2 & 6.15 $\pm$ 0.03 & 39621 $\pm$ 9.12 & -\\
    & 4 & 6.22 $\pm$ 0.06 & 16840 $\pm$ 0.11 & -\\
    \midrule
    \texttt{ResIST} & 2 & 5.99 $\pm$ 0.01 & 24058 $\pm$ 3.22 & \textbf{1.64$\times$ }\\
    & 4 & 6.69 $\pm$ 0.17 & 11264 $\pm$ 49.38 & \textbf{1.49$\times$}\\
 \bottomrule
\end{tabular}
\label{obj_detection}
\end{small}
\vspace{-0.3cm}
\end{table}

\vspace{-0.1cm}
\subsection{More experiments}
\vspace{-0.1cm}
In the Appendix A, we outline numerous ablation experiments that were performed using \texttt{ResIST}.
These experiments provide an understanding of the algorithm's behavior, as well as empirical support for its design: they include \texttt{ResIST} design decisions (Section \ref{design_ablation}), comparison of \texttt{ResIST} with ensemble methods (Section \ref{shallow_ensembles}), robustness to local iterations (Section \ref{S:local_iter}), applicability of \texttt{ResIST} to deeper architectures (Section \ref{S:deep_arch}), and compatibility to existing quantization/sparsification techniques (Section \ref{S:quant}).

\vspace{-0.3cm} 
\section{Conclusion}
\vspace{-0.1cm}
In the work, we present \texttt{ResIST}, a novel algorithm for synchronous, distributed training of ResNets.
\texttt{ResIST} operates by decomposing a global ResNet model into several shallower sub-ResNets, which are trained independently and itermittently aggregated into the global model. 
 By only communicating parameters of sub-ResNets between machines and training shallower, less expensive networks, \texttt{ResIST} reduces the communication and local training cost of synchronous, distributed training.
 We demonstrate the impact of \texttt{ResIST} on several image classification datasets, as well as in the object detection domain, by highlighting the significant training acceleration it provides in comparison to methods like local SGD \citep{use_local_sgd} without any deterioration in performance. 

We aim to extend \texttt{ResIST} to other network architectures, as \texttt{ResIST} is fully-extensible to all network architectures with residual connections.
Because residual connections are now standard in most important deep learning architectures (e.g., transformers), many opportunities to extend applications of \texttt{ResIST} exist.
On the other hand, \texttt{ResIST} has been shown to be fully-compatible with various gradient compression methods.
As such, we will investigate the prospect of fully integrating such compression methods within \texttt{ResIST}, both during training and communication phases, to further decrease memory and computation costs.

\clearpage
\bibliography{uai2022-resist}

\begin{thebibliography}{48}
\providecommand{\natexlab}[1]{#1}
\providecommand{\url}[1]{\texttt{#1}}
\expandafter\ifx\csname urlstyle\endcsname\relax
  \providecommand{\doi}[1]{doi: #1}\else
  \providecommand{\doi}{doi: \begingroup \urlstyle{rm}\Url}\fi

\bibitem[Abadi et~al.(2015)]{tensorflow}
Mart\'{\i}n Abadi et~al.
\newblock {TensorFlow}: Large-scale machine learning on heterogeneous systems,
  2015.

\bibitem[Aji and Heafield(2017)]{sparse-comm}
Alham~Fikri Aji and Kenneth Heafield.
\newblock Sparse communication for distributed gradient descent.
\newblock In \emph{Proceedings of the 2017 Conference on Empirical Methods in
  Natural Language Processing}, pages 440--445, 2017.

\bibitem[Akiba et~al.(2017)Akiba, Suzuki, and Fukuda]{15minimagenet}
Takuya Akiba, Shuji Suzuki, and Keisuke Fukuda.
\newblock Extremely large minibatch {SGD}: Training {ResNet}-50 on imagenet in
  15 minutes.
\newblock \emph{arXiv:1711.04325}, 2017.

\bibitem[Alistarh et~al.(2017)Alistarh, Grubic, Li, Tomioka, and
  Vojnovic]{commeff-sgd}
Dan Alistarh, Demjan Grubic, Jerry Li, Ryota Tomioka, and Milan Vojnovic.
\newblock {QSGD}: Communication-efficient {SGD} via gradient quantization and
  encoding.
\newblock \emph{Advances in Neural Information Processing Systems}, 30, 2017.

\bibitem[Assran et~al.(2020)Assran, Aytekin, Feyzmahdavian, Johansson, and
  Rabbat]{asynch-summary}
Mahmoud Assran, Arda Aytekin, Hamid~Reza Feyzmahdavian, Mikael Johansson, and
  Michael~G Rabbat.
\newblock Advances in asynchronous parallel and distributed optimization.
\newblock \emph{Proceedings of the IEEE}, 108\penalty0 (11):\penalty0
  2013--2031, 2020.

\bibitem[Assran and Rabbat(2020)]{assran2020asynchronous}
Mahmoud~S Assran and Michael~G Rabbat.
\newblock Asynchronous gradient push.
\newblock \emph{IEEE Transactions on Automatic Control}, 66\penalty0
  (1):\penalty0 168--183, 2020.

\bibitem[Ben-Nun and Hoefler(2019)]{parallelism_survey}
Tal Ben-Nun and Torsten Hoefler.
\newblock Demystifying parallel and distributed deep learning: An in-depth
  concurrency analysis.
\newblock \emph{ACM Computing Surveys (CSUR)}, 52\penalty0 (4):\penalty0 1--43,
  2019.

\bibitem[Chen et~al.(2018{\natexlab{a}})Chen, Yang, and
  Cheng]{multi_gpu_model_parallel}
Chi-Chung Chen, Chia-Lin Yang, and Hsiang-Yun Cheng.
\newblock Efficient and robust parallel dnn training through model parallelism
  on multi-gpu platform.
\newblock \emph{arXiv preprint arXiv:1809.02839}, 2018{\natexlab{a}}.

\bibitem[Chen et~al.(2018{\natexlab{b}})Chen, Rubanova, Bettencourt, and
  Duvenaud]{neuralode}
Ricky~TQ Chen, Yulia Rubanova, Jesse Bettencourt, and David~K Duvenaud.
\newblock Neural ordinary differential equations.
\newblock \emph{Advances in neural information processing systems}, 31,
  2018{\natexlab{b}}.

\bibitem[Deng et~al.(2009)Deng, Dong, Socher, Li, Li, and Fei-Fei]{imagenet}
Jia Deng, Wei Dong, Richard Socher, Li-Jia Li, Kai Li, and Li~Fei-Fei.
\newblock Imagenet: A large-scale hierarchical image database.
\newblock In \emph{2009 IEEE conference on computer vision and pattern
  recognition}, pages 248--255. Ieee, 2009.

\bibitem[Du et~al.(2019)Du, Lee, Li, Wang, and Zhai]{du2019gradient}
Simon Du, Jason Lee, Haochuan Li, Liwei Wang, and Xiyu Zhai.
\newblock Gradient descent finds global minima of deep neural networks.
\newblock In \emph{International conference on machine learning}, pages
  1675--1685. PMLR, 2019.

\bibitem[Dupont et~al.(2019)Dupont, Doucet, and Teh]{augmentedneuralode}
Emilien Dupont, Arnaud Doucet, and Yee~Whye Teh.
\newblock Augmented neural {ODE}s.
\newblock In \emph{Advances in Neural Information Processing Systems}, pages
  3140--3150, 2019.

\bibitem[Everingham et~al.(2010)Everingham, Van~Gool, Williams, Winn, and
  Zisserman]{pascalvoc}
Mark Everingham, Luc Van~Gool, Christopher~KI Williams, John Winn, and Andrew
  Zisserman.
\newblock The pascal visual object classes ({VOC}) challenge.
\newblock \emph{International journal of computer vision}, 88\penalty0
  (2):\penalty0 303--338, 2010.

\bibitem[Goyal et~al.(2017)Goyal, Doll{\'a}r, Girshick, Noordhuis, Wesolowski,
  Kyrola, Tulloch, Jia, and He]{1hrimagenet}
Priya Goyal, Piotr Doll{\'a}r, Ross Girshick, Pieter Noordhuis, Lukasz
  Wesolowski, Aapo Kyrola, Andrew Tulloch, Yangqing Jia, and Kaiming He.
\newblock Accurate, large minibatch {SGD}: Training {I}magenet in 1 hour.
\newblock \emph{arXiv:1706.02677}, 2017.

\bibitem[Guan et~al.(2019)Guan, Yin, Li, and Lu]{xpipe}
Lei Guan, Wotao Yin, Dongsheng Li, and Xicheng Lu.
\newblock {XPipe}: Efficient pipeline model parallelism for multi-{GPU DNN}
  training.
\newblock \emph{arXiv preprint arXiv:1911.04610}, 2019.

\bibitem[Gunther et~al.(2020)Gunther, Ruthotto, Schroder, Cyr, and
  Gauger]{layer_parallel_resnet}
Stefanie Gunther, Lars Ruthotto, Jacob~B Schroder, Eric~C Cyr, and Nicolas~R
  Gauger.
\newblock Layer-parallel training of deep residual neural networks.
\newblock \emph{SIAM Journal on Mathematics of Data Science}, 2\penalty0
  (1):\penalty0 1--23, 2020.

\bibitem[He et~al.(2016{\natexlab{a}})He, Zhang, Ren, and Sun]{preactres}
Kaiming He, Xiangyu Zhang, Shaoqing Ren, and Jian Sun.
\newblock Identity mappings in deep residual networks.
\newblock In \emph{European conference on computer vision}, pages 630--645.
  Springer, 2016{\natexlab{a}}.

\bibitem[He et~al.(2016{\natexlab{b}})He, Zhang, Ren, and Sun]{resnet}
Kaiming He, Xiangyu Zhang, Shaoqing Ren, and Jian Sun.
\newblock Deep residual learning for image recognition.
\newblock In \emph{Proceedings of the IEEE conference on computer vision and
  pattern recognition}, pages 770--778, 2016{\natexlab{b}}.

\bibitem[Huang et~al.(2016)Huang, Sun, Liu, Sedra, and Weinberger]{stochdepth}
Gao Huang, Yu~Sun, Zhuang Liu, Daniel Sedra, and Kilian~Q Weinberger.
\newblock Deep networks with stochastic depth.
\newblock In \emph{European conference on computer vision}, pages 646--661.
  Springer, 2016.

\bibitem[Huang et~al.(2019)Huang, Cheng, Bapna, Firat, Chen, Chen, Lee, Ngiam,
  Le, Wu, et~al.]{huang2019gpipe}
Yanping Huang, Youlong Cheng, Ankur Bapna, Orhan Firat, Dehao Chen, Mia Chen,
  HyoukJoong Lee, Jiquan Ngiam, Quoc~V Le, Yonghui Wu, et~al.
\newblock Gpipe: Efficient training of giant neural networks using pipeline
  parallelism.
\newblock \emph{Advances in neural information processing systems}, 32, 2019.

\bibitem[Jia et~al.(2018)]{4minimagenet}
Xianyan Jia et~al.
\newblock Highly scalable deep learning training system with mixed-precision:
  Training imagenet in four minutes.
\newblock \emph{arXiv preprint arXiv:1807.11205}, 2018.

\bibitem[Jiang and Agrawal(2018)]{linear-speed-quant}
Peng Jiang and Gagan Agrawal.
\newblock A linear speedup analysis of distributed deep learning with sparse
  and quantized communication.
\newblock \emph{Advances in Neural Information Processing Systems}, 31, 2018.

\bibitem[Kirby et~al.(2020)Kirby, Samsi, Jones, Reuther, Kepner, and
  Gadepally]{nonlinear_multigrid_layer_parallel}
Andrew Kirby, Siddharth Samsi, Michael Jones, Albert Reuther, Jeremy Kepner,
  and Vijay Gadepally.
\newblock Layer-parallel training with gpu concurrency of deep residual neural
  networks via nonlinear multigrid.
\newblock In \emph{2020 IEEE High Performance Extreme Computing Conference
  (HPEC)}, pages 1--7. IEEE, 2020.

\bibitem[Koloskova et~al.(2019)Koloskova, Stich, and
  Jaggi]{koloskova2019decentralized}
Anastasia Koloskova, Sebastian Stich, and Martin Jaggi.
\newblock Decentralized stochastic optimization and gossip algorithms with
  compressed communication.
\newblock In \emph{International Conference on Machine Learning}, pages
  3478--3487. PMLR, 2019.

\bibitem[Koloskova et~al.(2020)Koloskova, Loizou, Boreiri, Jaggi, and
  Stich]{koloskova2020unified}
Anastasia Koloskova, Nicolas Loizou, Sadra Boreiri, Martin Jaggi, and Sebastian
  Stich.
\newblock A unified theory of decentralized sgd with changing topology and
  local updates.
\newblock In \emph{International Conference on Machine Learning}, pages
  5381--5393. PMLR, 2020.

\bibitem[Li et~al.(2017)Li, Kadav, Durdanovic, Samet, and Graf]{filterprune}
Hao Li, Asim Kadav, Igor Durdanovic, Hanan Samet, and Hans~Peter Graf.
\newblock Pruning filters for efficient convnets.
\newblock In \emph{5th International Conference on Learning Representations,
  {ICLR} 2017}, 2017.

\bibitem[{Lin} et~al.(2018){Lin}, {Stich}, {Kshitij Patel}, and
  {Jaggi}]{use_local_sgd}
Tao {Lin}, Sebastian~U. {Stich}, Kumar {Kshitij Patel}, and Martin {Jaggi}.
\newblock {Don't Use Large Mini-Batches, Use Local SGD}.
\newblock art. arXiv:1808.07217, August 2018.

\bibitem[Lu et~al.(2018)Lu, Zhong, Li, and Dong]{beyondfinite}
Yiping Lu, Aoxiao Zhong, Quanzheng Li, and Bin Dong.
\newblock Beyond finite layer neural networks: Bridging deep architectures and
  numerical differential equations.
\newblock In \emph{International Conference on Machine Learning}, pages
  3276--3285. PMLR, 2018.

\bibitem[Lu et~al.(2020)Lu, Ma, Lu, Lu, and Ying]{mfresnet}
Yiping Lu, Chao Ma, Yulong Lu, Jianfeng Lu, and Lexing Ying.
\newblock A mean field analysis of deep resnet and beyond: Towards provably
  optimization via overparameterization from depth.
\newblock In \emph{International Conference on Machine Learning}, pages
  6426--6436. PMLR, 2020.

\bibitem[McMahan et~al.(2017)McMahan, Moore, Ramage, Hampson, and
  y~Arcas]{fed_avg}
Brendan McMahan, Eider Moore, Daniel Ramage, Seth Hampson, and Blaise~Aguera
  y~Arcas.
\newblock Communication-efficient learning of deep networks from decentralized
  data.
\newblock In \emph{Artificial intelligence and statistics}, pages 1273--1282.
  PMLR, 2017.

\bibitem[Nedic and Ozdaglar(2009)]{nedic2009distributed}
Angelia Nedic and Asuman Ozdaglar.
\newblock Distributed subgradient methods for multi-agent optimization.
\newblock \emph{IEEE Transactions on Automatic Control}, 54\penalty0
  (1):\penalty0 48--61, 2009.

\bibitem[Paszke et~al.(2019)]{pytorch}
Adam Paszke et~al.
\newblock Pytorch: An imperative style, high-performance deep learning library.
\newblock In \emph{Advances in Neural Information Processing Systems 32}, pages
  8024--8035. 2019.

\bibitem[Recht et~al.(2019)Recht, Roelofs, Schmidt, and
  Shankar]{recht2019imagenet}
Benjamin Recht, Rebecca Roelofs, Ludwig Schmidt, and Vaishaal Shankar.
\newblock Do imagenet classifiers generalize to imagenet?
\newblock In \emph{International Conference on Machine Learning}, pages
  5389--5400. PMLR, 2019.

\bibitem[Redmon and Farhadi(2017)]{yolov2}
Joseph Redmon and Ali Farhadi.
\newblock {YOLO9000}: better, faster, stronger.
\newblock In \emph{Proceedings of the IEEE conference on computer vision and
  pattern recognition}, pages 7263--7271, 2017.

\bibitem[Sergeev and Del~Balso(2018)]{sergeev2018horovod}
Alexander Sergeev and Mike Del~Balso.
\newblock Horovod: fast and easy distributed deep learning in {T}ensorflow.
\newblock \emph{arXiv:1802.05799}, 2018.

\bibitem[Shi et~al.(2018)Shi, Wang, and Chu]{distrib-benchmark}
Shaohuai Shi, Qiang Wang, and Xiaowen Chu.
\newblock Performance modeling and evaluation of distributed deep learning
  frameworks on {GPU}s.
\newblock In \emph{2018 IEEE 16th Intl Conf on Dependable, Autonomic and Secure
  Computing}, pages 949--957. IEEE, 2018.

\bibitem[Stich(2019)]{localsgdconverge}
Sebastian~U. Stich.
\newblock Local {SGD} converges fast and communicates little.
\newblock In \emph{International Conference on Learning Representations}, 2019.

\bibitem[Tang et~al.(2018)Tang, Gan, Zhang, Zhang, and Liu]{comm-comp}
Hanlin Tang, Shaoduo Gan, Ce~Zhang, Tong Zhang, and Ji~Liu.
\newblock Communication compression for decentralized training.
\newblock \emph{Advances in Neural Information Processing Systems}, 31, 2018.

\bibitem[Thorpe and van Gennip(2018)]{deeplimit}
Matthew Thorpe and Yves van Gennip.
\newblock Deep limits of residual neural networks.
\newblock \emph{arXiv:1810.11741}, 2018.

\bibitem[Veit et~al.(2016)Veit, Wilber, and Belongie]{shallowensemble}
Andreas Veit, Michael~J Wilber, and Serge Belongie.
\newblock Residual networks behave like ensembles of relatively shallow
  networks.
\newblock \emph{Advances in neural information processing systems}, 29, 2016.

\bibitem[Wangni et~al.(2018)Wangni, Wang, Liu, and Zhang]{grad-sparse}
Jianqiao Wangni, Jialei Wang, Ji~Liu, and Tong Zhang.
\newblock Gradient sparsification for communication-efficient distributed
  optimization.
\newblock \emph{Advances in Neural Information Processing Systems}, 31, 2018.

\bibitem[You et~al.()You, Li, Reddi, Hseu, Kumar, Bhojanapalli, Song, Demmel,
  Keutzer, and Hsieh]{76minbert}
Yang You, Jing Li, Sashank~J. Reddi, Jonathan Hseu, Sanjiv Kumar, Srinadh
  Bhojanapalli, Xiaodan Song, James Demmel, Kurt Keutzer, and Cho{-}Jui Hsieh.
\newblock Large batch optimization for deep learning: Training {BERT} in 76
  minutes.
\newblock In \emph{8th International Conference on Learning Representations,
  {ICLR} 2020}.

\bibitem[You et~al.(2017)You, Gitman, and Ginsburg]{32Kbatch}
Yang You, Igor Gitman, and Boris Ginsburg.
\newblock Scaling {SGD} batch size to 32{K} for {I}magenet training.
\newblock \emph{arXiv preprint arXiv:1708.03888}, 6\penalty0 (12):\penalty0 6,
  2017.

\bibitem[You et~al.(2018)You, Zhang, Hsieh, Demmel, and
  Keutzer]{you2018imagenet}
Yang You, Zhao Zhang, Cho-Jui Hsieh, James Demmel, and Kurt Keutzer.
\newblock Imagenet training in minutes.
\newblock In \emph{Proceedings of the 47th International Conference on Parallel
  Processing}, pages 1--10, 2018.

\bibitem[Yu et~al.(2019)Yu, Wu, and Huang]{double-quant}
Yue Yu, Jiaxiang Wu, and Longbo Huang.
\newblock Double quantization for communication-efficient distributed
  optimization.
\newblock \emph{Advances in Neural Information Processing Systems}, 32, 2019.

\bibitem[{Zhang} et~al.(2016){Zhang}, {De Sa}, {Mitliagkas}, and
  {R{\'e}}]{parallel_sgd}
Jian {Zhang}, Christopher {De Sa}, Ioannis {Mitliagkas}, and Christopher
  {R{\'e}}.
\newblock {Parallel SGD: When does averaging help?}
\newblock art. arXiv:1606.07365, June 2016.

\bibitem[Zhu et~al.(2020)Zhu, Zhao, Li, Roth, Xu, and Xu]{lamp}
Wentao Zhu, Can Zhao, Wenqi Li, Holger Roth, Ziyue Xu, and Daguang Xu.
\newblock {Lamp}: Large deep nets with automated model parallelism for image
  segmentation.
\newblock In \emph{International Conference on Medical Image Computing and
  Computer-Assisted Intervention}, pages 374--384. Springer, 2020.

\bibitem[Zinkevich et~al.(2010)Zinkevich, Weimer, Li, and Smola]{parallelsgd}
Martin Zinkevich, Markus Weimer, Lihong Li, and Alex~J Smola.
\newblock Parallelized stochastic gradient descent.
\newblock In \emph{Advances in neural information processing systems}, pages
  2595--2603, 2010.

\end{thebibliography}

\appendix
\providecommand{\upGamma}{\Gamma}
\providecommand{\uppi}{\pi}
\onecolumn
\allowdisplaybreaks

\section{Ablations}
These experiments provide an understanding of the algorithm's behavior, as well as empirical support for its design.

\subsection{Designing \texttt{ResIST}} \label{design_ablation}
Extensive ablation experiments are conducted on the CIFAR10 dataset, outlined in Fig. \ref{resist_ablations}, to empirically motivate the design choices made within \texttt{ResIST} (i.e., see Sec. \ref{S:supp_tech}).
For the two sub-ResNet case, the naive implementation of \texttt{ResIST}, which evenly splits all convolutional blocks between subnetworks, is shown to perform poorly (i.e., $<$70\% on CIFAR10).
The accuracy of \texttt{ResIST} is improved over 25\% by only allowing select layers to be partitioned and ensuring activations are scaled correctly when performing inference with the full network.
The pre-activation ResNet is shown to yield an improvement in accuracy, leading \texttt{ResIST} to perform near optimally with two sub-ResNets.

\begin{figure}[!htp]
\centering
\includegraphics[width=5in]{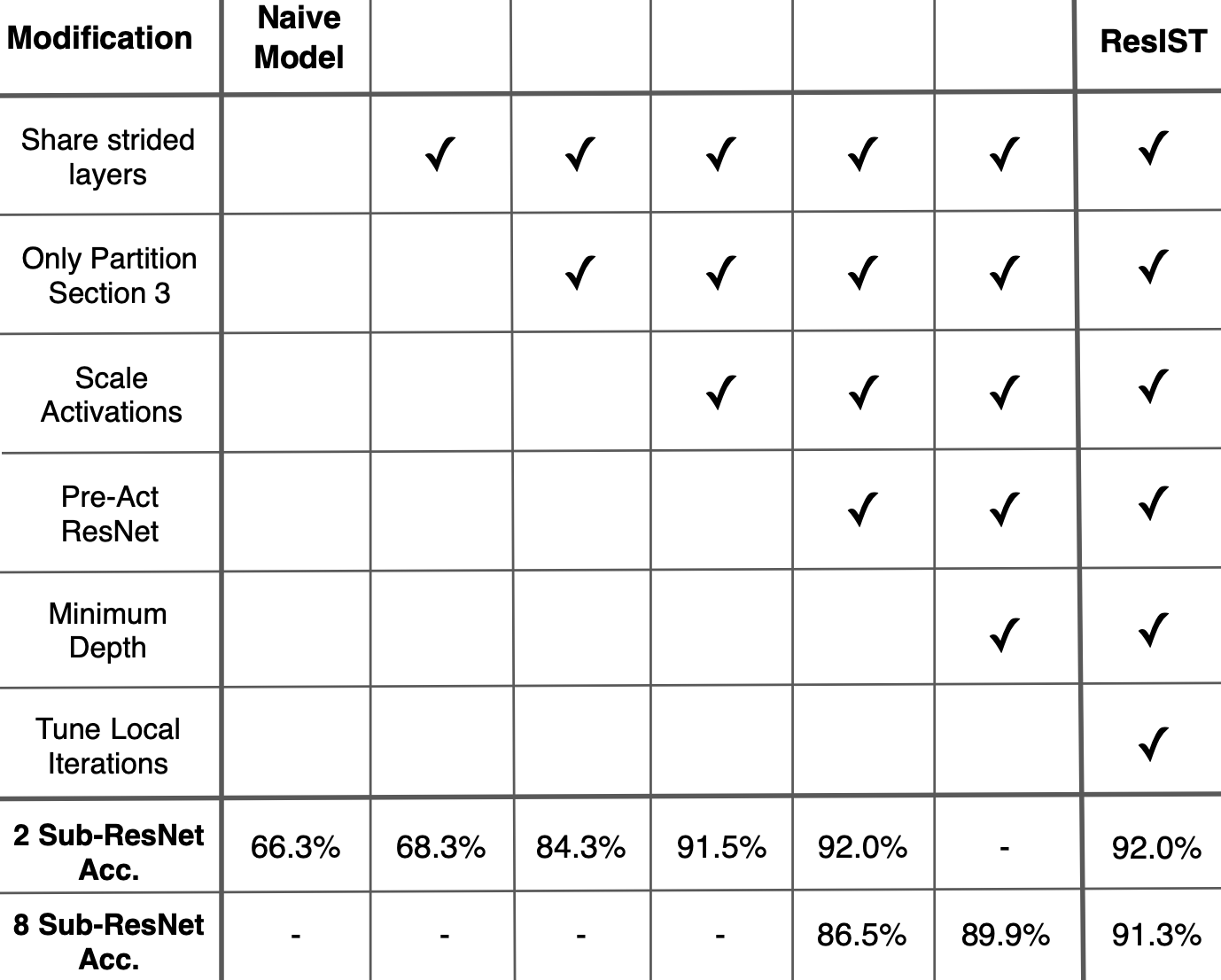} 
\caption{Test accuracies on the CIFAR10 dataset for a single run for the major ablation experiments performed with \texttt{ResIST}.}
\label{resist_ablations}

\end{figure}

When \texttt{ResIST} is expanded to eight sub-ResNets, we initially observe a significant decrease in model accuracy.
However, as can be seen in Fig. \ref{resist_ablations}, this gap can be closed by enforcing a minimum depth on sub-ResNets and tuning the number of local iterations.
By making these extra modifications, \texttt{ResIST} begins to perform similarly with two to eight sub-ResNets, yielding compelling performance.

\subsection{Shallow Ensembles}\label{shallow_ensembles}




The \texttt{ResIST} algorithm requires that independently-trained sub-ResNets must have their parameters synchronized intermittently. 
Such synchronization, however, can be completely avoided by training each sub-ResNet separately and forming an ensemble (i.e., \texttt{ResIST} without any aggregation).
Although maintaining an ensemble has several drawbacks (e.g., slower inference, more parameters, etc.), the training time of the ensemble would nonetheless be reduced in comparison to \texttt{ResIST} by avoiding communication altogether. 
Therefore, the performance of such an ensemble should be compared to the models trained with \texttt{ResIST}.

\begin{table}[!ht]
\centering
\caption{Performance of indpendently-trained ensembles of shallow ResNets in comparison to \texttt{ResIST} on CIFAR10 and CIFAR100 (denoted as C10 and C100, respectively).}
 \vspace{0.2cm}
 \setlength{\tabcolsep}{.2\tabcolsep}
\begin{tabular}{ccccccccc}
\toprule
    Dataset & Method & & 2 Model & &  4 Model & & 8 Model \\ \midrule
    C10 & Ensemble & & 92.27 \% $\pm$ 0.00 & &  92.56\% $\pm$ 0.03 & & 90.67 \% $\pm$ 0.04 \\    
    & \texttt{ResIST} & & 91.95\% $\pm$ 0.32 & & 92.35\% $\pm$ 0.22 & & 91.45\% $\pm$ 0.30\\
    \midrule
    C100 & Ensemble & & 72.08\% $\pm$ 0.05 & & 72.12\% $\pm$ 0.04 & & 67.98 \% $\pm$ 0.12 \\
    & \texttt{ResIST} & & 70.06\% $\pm$ 0.51 & & 71.30\% $\pm$ 0.20 & & 70.26\% $\pm$ 0.21 \\
 \bottomrule
\end{tabular}
\label{ensemble_perf}
\end{table}

\begin{table*}
\centering
\caption{Test accuracy on CIFAR10 (C10) and CIFAR100 (C100) for deeper architectures trained with \texttt{ResIST} and local SGD (LSGD). All tests were performed with 100 local iterations between synchronization rounds. All models were trained for 80 epochs.}
\vspace{0.1cm}
\begin{small}
\begin{tabular}{ccc|ccc|ccc}
\toprule
     &&& \multicolumn{3}{c|}{ResNet152} & \multicolumn{3}{c}{ResNet200}\\
     Dataset & \# Machines & Method & Time & Test Acc. & Speedup & Time & Test Acc. & Speedup  \\
     \midrule
     C10 & 2 & LSGD & 3512s & 92.27\% $\pm$ 0.003 & & 4575s & 92.31\% $\pm$ 0.001 & \\
     && \texttt{ResIST} & 2215s & 92.01\% $\pm$ 0.002 & \textbf{1.58$\times$} & 2380s & 92.10\% $\pm$ 0.001 & \textbf{1.92$\times$} \\
     & 4 & LSGD & 3598s & 91.39\% $\pm$ 0.001 &  & 4357s & 91.35\% $\pm$ 0.000 & \\
     && \texttt{ResIST} & 1054s & 90.67\% $\pm$ 0.001 & \textbf{3.41$\times$} &  1161s & 90.27\% $\pm$ 0.001 & \textbf{3.75$\times$} \\
     \midrule
      C100 & 2 & LSGD & 3528s & 70.50\% $\pm$ 0.003 & & 4639s & 71.05\% $\pm$ 0.005 & \\
      && \texttt{ResIST} & 2291s & 70.32\% $\pm$ 0.005 & \textbf{1.53$\times$} & 2202s & 70.71\% $\pm$ 0.002 & \textbf{2.10$\times$} \\
      & 4 & LSGD & 3518s & 68.39\% $\pm$ 0.004 & & 4391s & 69.05\% $\pm$ 0.003 & \\
      && \texttt{ResIST} & 1164s & 67.27\% $\pm$ 0.003 & \textbf{3.02$\times$} & 1195s & 67.62\% $\pm$ 0.001 & \textbf{3.67$\times$} \\
 \bottomrule
\end{tabular}
\end{small}
\label{tab:deep_net_results}
\end{table*}

The performance of sub-ResNet ensembles in comparison to models trained with \texttt{ResIST} is displayed in Table \ref{ensemble_perf}.
For 8 Sub-ResNets, the shallow ensembles achieve inferior performance in comparison to \texttt{ResIST}.
When two and four Sub-ResNets are used, the performance of shallow ensembles and \texttt{ResIST} is comparable (i.e., $<1\%$ performance difference in most cases).
However, it should be noted that such shallow ensembles of two or four sub-ResNets, in comparison to \texttt{ResIST}, cause a $2\times$ to $4\times$ slowdown in inference time (i.e., inference time for a single Sub-ResNet is not significantly faster than that of the global ResNet).
Furthermore, the ensembles consume more parameters in comparison to global ResNet trained with \texttt{ResIST}. 

\subsection{Robustness to Local Iterations}
\label{S:local_iter}

\begin{figure}[!htp]
    \centering
    \includegraphics[width=0.7\linewidth]{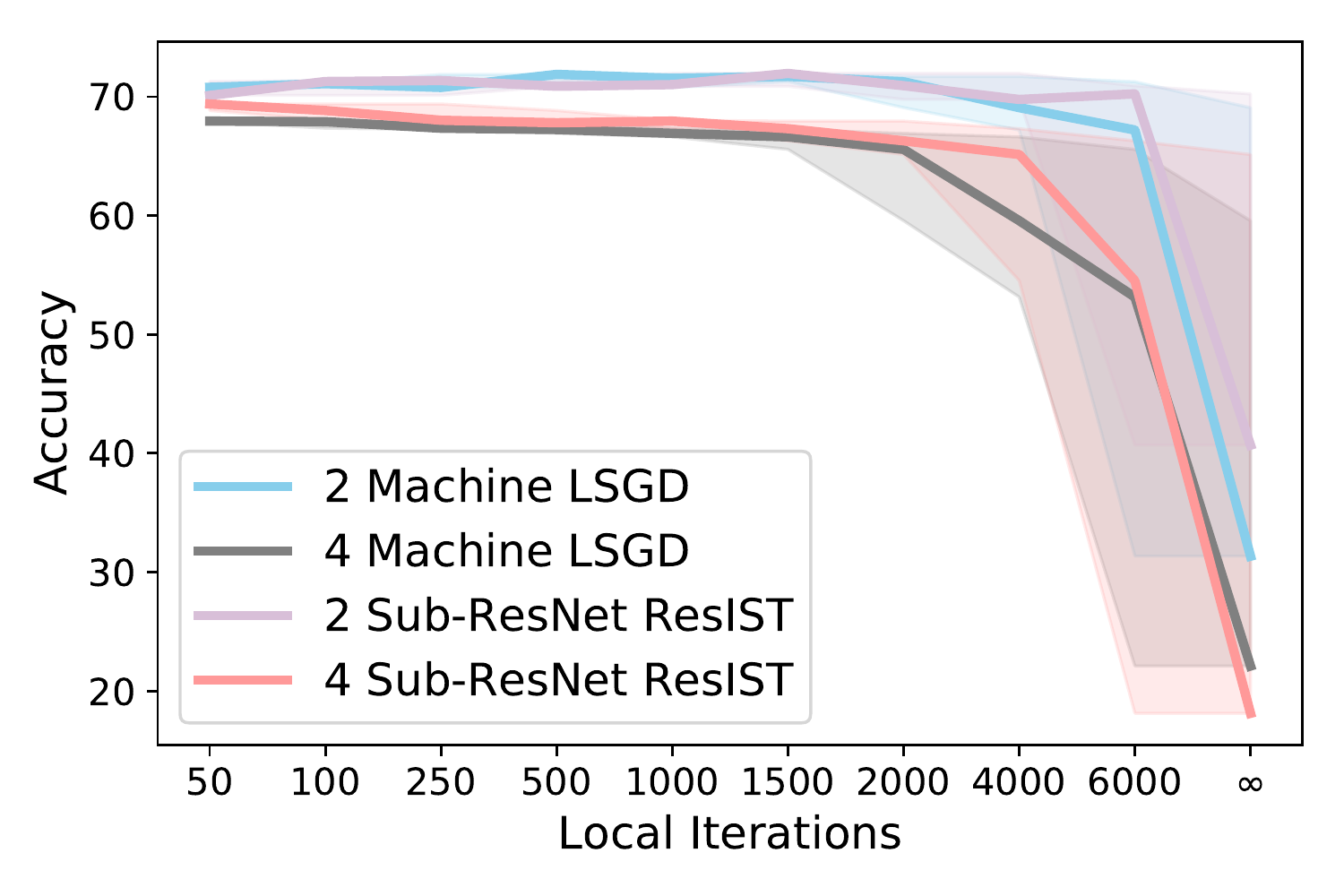} \vspace{-0.4cm}
    \caption{Test accuracy on CIFAR100 for ResNet-101 trained with both \texttt{ResIST} and local SGD (LSGD) with different numbers of local iterations. $\infty$ local iterations refers to aggregating parameters only once at the end of training (i.e., single-shot averaging). Shaded regions reflect deviations in accuracy.}
    \label{fig:local_iter}
    \vspace{-0.5cm}
\end{figure}

\texttt{ResIST} is robust to various numbers of local iterations \citep{use_local_sgd, parallel_sgd, fed_avg}.
An extensive sweep over possible values of $\ell$ is performed on CIFAR100.
The results of this experiment are depicted in Fig. \ref{fig:local_iter}.
As can be seen, \texttt{ResIST} achieves high accuracy even with thousands of local SGD iterations (i.e., previous work typically uses much fewer \citep{use_local_sgd}).
However, if more sub-ResNets are used, performance tends to deteriorate more quickly as local iterations increase.
Due to the robustness of \texttt{ResIST} to large numbers of local iterations, training can be accelerated without deteriorating model performance by simply increasing the value of $\ell$.
Local SGD was found to demonstrate similar robustness to the number of local iterations, as shown in Fig. \ref{fig:local_iter}.

\subsection{Deeper architectures}
\label{S:deep_arch}
The \texttt{ResIST} methodology is easily applicable to deeper architectures.
To demonstrate this, results are replicated for CIFAR10 and CIFAR100 datasets with ResNet152 and ResNet200.
These deeper architectures are identical to the original ResNet101 architecture (i.e., see Fig. \ref{model_depict}).
However, more residual blocks are added to the third section of the ResNet (i.e., the highlighted portion of Fig. \ref{model_depict}) to increase the model's depth.
It should be noted that convolutional blocks within the third section of the ResNet are partitioned in \texttt{ResIST} by default (see Sec. \ref{subnet_sec}). 
As a result, all extra residual blocks within these deeper architectures are partitioned to sub-ResNets by \texttt{ResIST} (i.e., no extra blocks are shared between sub-ResNets), allowing \texttt{ResIST} to achieve greater acceleration in comparison to local SGD.

The results of experiments with deeper ResNets are presented in Table \ref{tab:deep_net_results}. 
\texttt{ResIST} performs competitively with localSGD in all cases.
Furthermore, \texttt{ResIST} achieves a significant speedup in comparison to local SGD that becomes more pronounced as the model becomes deeper.
E.g., for 4-GPUs, \texttt{ResIST} completes training $>3 \times$ faster than local SGD for ResNet200 on both datasets.
This speedup is caused by a greater ratio of total network blocks being partitioned to sub-ResNets in \texttt{ResIST}.
While local SGD must communicate all parameters between machines, \texttt{ResIST} achieves a relative decrease in communication by partitioning all extra residual blocks evenly between sub-ResNets.

\subsection{\texttt{ResIST} and Quantization/Sparse Gradients}
\label{S:quant}

Many quantization \citep{commeff-sgd, double-quant} and sparsification \citep{sparse-comm, linear-speed-quant} techniques have been proposed for reducing communication costs in distributed training.
Such techniques focus on compressing communicated data, and they do not interfere with our methodology, which provides a novel approach to model synchronization and training.
The proposed approach can be easily combined with existing compression techniques to further reduce communication costs and accelerate training \emph{with no extra tuning or modifications}.
To demonstrate that \texttt{ResIST} works well with quantization, we compress all communicated parameters using both four-bit and eight-bit compression.
Table \ref{quantization} shows that \texttt{ResIST} retains its performance until the compression level reaches five-bit and lower.
We also perform experiments with sparsification of communicated weights by only keeping 25\% of total weights within each synchronization round. 
Such a strategy reaches a validation peformance of 71.25\% on CIFAR100.
We summarize the results of all quantization experiments in Fig. \ref{fig:budget}, where we compare communication budgets across different compression techniques with \texttt{ResIST}.
From this figure, it is clear that \texttt{ResIST} is most efficient with six-bit quantization and is compatible with most main-stream compression techniques.


\begin{table}[!htp]
\centering
\caption{Test Accuracy for \texttt{ResIST} combined with quantization on CIFAR10 and CIFAR100 (denoted as C10 and C100).}
\vspace{0.1cm}
\setlength{\tabcolsep}{.5\tabcolsep}
\begin{tabular}{cccccc}
\toprule
    Dataset  & 8 bit & 7 bit & 6 bit & 5 bit & 4 bit\\ \midrule
    C10   & 92.14\% & 92.26\% & 91.91\% & 91.35\%  & 76.33\%\\
    C100  & 71.38\% & 72.15\% & 71.37\% & 68.29\% & 40.48\%\\
  \bottomrule
 \end{tabular}
 \label{quantization}
 \end{table}
 
 \begin{figure}
    \centering
    \includegraphics[width=0.7\linewidth]{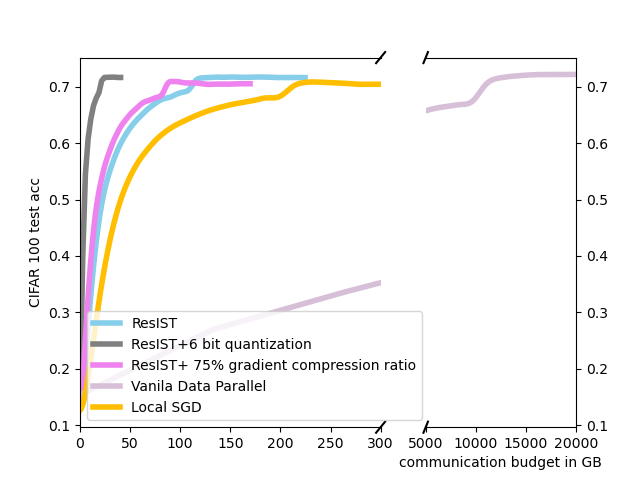}
    \caption{Test accuracy vs. communication budget for \texttt{ResIST}, \texttt{ResIST}+quantization, \texttt{ResIST}+gradient compression, local SGD and vanilla data parallel on CIFAR100. All models are trained over a 4-GPU cluster.}
    \label{fig:budget}
    
\end{figure}

\section{Proof for \texttt{ResIST}}{\label{sec:proof}}
Suppose we have $S$ workers, for subnetwork $v$ at local training step $l_t \le \ell$ and global synchronization step $t \le T$:

\begin{align*}
\vect{x}^{(1)}_{v,l_t,t}&=\sqrt{\frac{c_{\sigma}}{m}} \relu{\mat{W}^{(1)}_{v,l_t,t}\vect{x_{v,l_t,t}}}, \nonumber\\
\vect{x}^{(h)}_{v,l_t,t} & =\vect{x}^{(h-1)}_{v,l_t,t}+\frac{c_{res}}{H\sqrt{m}} \relu{\mat{W}^{(h)}_{v,l_t,t}\vect{x}^{(h-1)}_{v,l_t,t}}M^{(h)}_{v,t} \nonumber\\
& ~~~~~~~~~~~~~~~~~~~~~~~~~~~~~~~\text{ for } 2\le h\le H, \nonumber\\
f_{res}(\vect{x},\params)&=\vect{a}_{v,l_t,t}^\top \vect{x}^{(H)}_{v,l_t,t}
\end{align*}
where  $0< c_{res} < 1$ is a small constant and $M^{(h)}_{v,t}$ is random binary variable in layer dropout or the indicator in \texttt{ResIST} that indicates whether this layer is partitioned to this subnetwork. such mask variable is constant during local training steps and re-sampled/re-assigned at global synchronization step. In \texttt{ResIST} and other research on layer dropout for ResNet, last layer is never dropped/paritioned but shared with all workers. Thus, in the following proof, we will follow this setting.

Note here we use a $\frac{c_{res}}{H\sqrt{m}}$ scaling.

The gradient for subnetwork is 
\begin{align*}
\frac{\partial L}{\partial \mat{W}^{(h)}_{v,l_t,t}} =& \frac{c_{res}}{H\sqrt{m}}
\sum_{i=1}^{n}(y_i-u_i)\vect{x}_{i,v,l_t,t}^{(h-1)} \cdot
\left[\vect{a}_{v,l_t,t}^\top \prod_{l=h+1}^{H}\left(\mat{I}+\frac{c_{res}}{H\sqrt{m}}\mat{J}_{i,v,l_t,t}^{(l)}\mat{W}_{v,l_t,t}^{(l)} M^{(l)}_{v,t} \right) \mat{J}_{i,v,l_t,t}^{(h)}  M^{(h)}_{v,t}\right]
\end{align*}
For subnetwork, $\mat{G}^{(H)}$ has the same form as in layer drop ResNet.

The accumulated gradients of all the subnetworks:
\begin{align*}
\mathcal{W}^{(h)}_{t+1} - \mathcal{W}^{(h)}_{t} &= 
\eta \frac{\sum_{v=1}^{S} \sum_{l_t=1}^{\ell} \frac{\partial L}{\partial \mat{W}^{(h)}_{v,l_t,t}}}{\sum_{v=1}^{S}M^{(h)}_{v,t}} \\
&=\frac{\eta}{\sum_{v=1}^{S}M^{(h)}_{v,t}} \sum_{v=1}^{S} \sum_{l_t=1}^{\ell}\frac{c_{res}}{H\sqrt{m}}
\sum_{i=1}^{n}(y_i-u_i)\vect{x}_{i,v,l_t,t}^{(h-1)} \cdot
\left[\vect{a}_{v,l_t,t}^\top \prod_{l=h+1}^{H}\left(\mat{I}+\frac{c_{res}}{H\sqrt{m}}\mat{J}_{i,v,l_t,t}^{(l)}\mat{W}_{v,l_t,t}^{(l)} M^{(l)}_{v,t} \right) \mat{J}_{i,v,l_t,t}^{(h)}  M^{(h)}_{v,t}\right]
\end{align*}

The whole network at global synchronization step t+1

\begin{align*}
\vect{x}^{(1)}_{t}&=\sqrt{\frac{c_{\sigma}}{m}} \relu{\frac{\sum_{v=1}^{S} \mat{W}^{(1)}_{v,\ell,t}}{S}\vect{x_{t}}}, \nonumber\\
\vect{x}^{(h)}_{t} & =\vect{x}^{(h-1)}_{t}+\frac{c_{res}}{H\sqrt{m}} \relu{\frac{\sum_{v=1}^{S} \mat{W}^{(h)}_{v,\ell,t}M^{(h)}_{v,t}}{\sum_{v=1}^{S}M^{(h)}_{v,t}}\vect{x}^{(h-1)}_{t}} \nonumber\\
& ~~~~~~~~~~~~~~~~~~~~~~~~~~~~~~~\text{ for } 2\le h\le H, \nonumber\\
f_{res}(\vect{x},\params)&=\frac{\sum_{v=1}^{S} \vect{a}_{v,\ell,t}}{S}^\top \vect{x}^{(H)}_{t}
\end{align*}
Let $\mathcal{W}^{(h)}_t=\frac{\sum_{v=1}^{S} \mat{W}^{(h)}_{v,\ell,t}M^{(h)}_{v,t}}{\sum_{v=1}^{S}M^{(h)}_{v,t}}$, $\vect{a}_t=\frac{\sum_{v=1}^{S} \vect{a}_{v,\ell,t}}{S}$

The whole network at global synchronization step 0

\begin{align*}
\vect{x}^{(1)}_{0}&=\sqrt{\frac{c_{\sigma}}{m}} \relu{\mat{W}^{(1)}_{0}\vect{x}}, \nonumber\\
\vect{x}^{(h)}_{0} & =\vect{x}^{(h-1)}_{0}+\frac{c_{res}}{H\sqrt{m}} \relu{\mat{W}^{(h)}_{0}\vect{x}^{(h-1)}_{0}} \nonumber\\
& ~~~~~~~~~~~~~~~~~~~~~~~~~~~~~~~\text{ for } 2\le h\le H, \nonumber\\
f_{res}(\vect{x},\params)&=\vect{a}_0^\top \vect{x}^{(H)}_{0}
\end{align*}

\subsection{Proof Sketch}
We can write the loss of the whole network at global syncrhonization step t+1 as \[
L(\params(t),\mat{M}_{1,t},\mat{M}_{2,t}...\mat{M}_{S,t}) = \frac{1}{2}\norm{\vect{y}-\vect{u}(t, \mat{M}_{1,t},\mat{M}_{2,t}...\mat{M}_{S,t})}_2^2.
\]
where $\mat{M}_{v,t}=\{ M^{(1)}_{v,t},M^{(2)}_{v,t}...M^{(H)}_{v,t} \}$
Let $\mathcal{M}_t=\{ \mat{M}_{1,t},\mat{M}_{2,t}...\mat{M}_{S,t} \}$

For convience, we drop all mask notation in the following proof. Let $\hat{u}(t)$ be the output of the whole network at global synchronization step $t+1$.
Now recall the progress of loss function:
\begin{align*}
\norm{\vect{y}-\vect{\hat{u}}(t+1)}_2^2  
= & \norm{\vect{y}-\vect{\hat{u}}(t)}_2^2 -2 \left(\vect{y}-\vect{\hat{u}}(t)\right)^\top \left(\vect{\hat{u}}(t+1)-\vect{\hat{u}}(t)\right) + \norm{\vect{\hat{u}}(t+1)-\vect{\hat{u}}(t)}_2^2
\end{align*}
Following \cite{du2019gradient}, we apply Taylor expansion on $\left(\vect{\hat{u}}(t+1)-\vect{\hat{u}}(t)\right)$ and look at the $i$th coordinate.
\begin{align*}
\hat{u}_{i}(t+1)-\hat{u}_{i}(t) & = -\langle \theta(t+1)-\theta(t), \hat{u}'_i\left(\params(t)\right) \rangle + \int_{s=0}^{1} \langle \theta(t+1)-\theta(t), \hat{u}'_i\left(\params(t)\right) -\hat{u}'_i\left(\params(t)-s (\params(t)-\params(t+1)) \right) \rangle ds \\
& \triangleq I^i_1(t)+I_2^i(t)
\end{align*}
However, it is not obvious that $I_1(t)$ and $I_2(t)$ can be directed bounded to show the decrease of the loss of the whole network as both of them involve the accumulated gradient change from distributed local subnetwork training. Thus, we introduce a new term $I'^i_1(t)$ as below, which relates to the hypothetical global gradient direction as if the whole network trained centrally.  
\begin{align*}
	I'^i_1(t)=&-\eta \ell \langle L'(\params(t)), \hat{u}'_i\left(\params(t)\right) \rangle \\
	=&-\eta \ell \sum_{j=1}^{n}(\hat{u}_j-y_j) \langle \hat{u}'_j(\params(t)), \hat{u}'_i\left(\params(t)\right) \rangle\\
	\triangleq& -\eta \ell \sum_{j=1}^{n}(\hat{u}_j-y_j) \sum_{h=1}^{H+1}\mat{\hat{G}}^{(h)}_{ij}(t)
	\end{align*}
Accordingly, 
\begin{align*}
&\norm{\vect{y}-\vect{\hat{u}}(t+1)}_2^2  \\
= & \norm{\vect{y}-\vect{\hat{u}}(t)}_2^2 -2 \left(\vect{y}-\vect{\hat{u}}(t)\right)^\top \left( \vect{I}_1(t)+\vect{I}_2(t)+\vect{I'}_1(t)-\vect{I'}_1(t) \right) + \norm{\vect{\hat{u}}(t+1)-\vect{\hat{u}}(t)}_2^2\\
= &\norm{\vect{y}-\vect{\hat{u}}(t)}_2^2 - 2 \left(\vect{y}-\vect{\hat{u}}(t)\right)^\top\vect{I'}_1(t) +2 \left(\vect{y}-\vect{\hat{u}}(t)\right)^\top(\vect{I'}_1(t)-\vect{I}_1(t))-2\left(\vect{y}-\vect{\hat{u}}(t)\right)^\top\vect{I}_2(t)+  \norm{\vect{\hat{u}}(t+1)-\vect{\hat{u}}(t)}_2^2 \\
\le & \left(1-\eta \ell \lambda_{\min}\left(\mat{\hat{G}}^{(H)}(t)\right)\right)\norm{\vect{y}-\vect{\hat{u}}(t)}_2^2 +2 \left(\vect{y}-\vect{\hat{u}}(t)\right)^\top(\vect{I'}_1(t)-\vect{I}_1(t))\\
&-2\left(\vect{y}-\vect{\hat{u}}(t)\right)^\top\vect{I}_2(t)+  \norm{\vect{\hat{u}}(t+1)-\vect{\hat{u}}(t)}_2^2.
\end{align*}



Our hypothesis is:
\begin{condition}\label{cond:linear_converge_resist}
	At the $t+1$-th global synchronization, for the whole network, we have \begin{align*}
	\norm{\vect{y}-\vect{\hat{u}}(t,\mathcal{M}_t)}_2^2 \le (1-\frac{\eta \ell \lambda_0}{2})^{t} \norm{\vect{y}-\vect{\hat{u}}(0)}_2^2.
	\end{align*}
\end{condition}
In order to prove this, we need to show $2 \left(\vect{y}-\vect{\hat{u}}(t)\right)^\top(\vect{I'}_1(t)-\vect{I}_1(t))$,$-2\left(\vect{y}-\vect{\hat{u}}(t)\right)^\top\vect{I}_2(t)$ and  $\norm{\vect{\hat{u}}(t+1)-\vect{\hat{u}}(t)}_2^2$  are proportional to $\eta^2 \norm{\vect{y}-\vect{\hat{u}}(t)}_2^2$
so if we set $\eta$ sufficiently small, this term is smaller than $\eta \lambda_{\min}\left(\mat{\hat{G}}^{(H)}(t)\right)\norm{\vect{y}-\vect{\hat{u}}(t)}_2^2$ and thus the loss function decreases with a linear rate. 

Further, similar to \cite{du2019gradient}, to prove the induction hypothesis, it suffices to prove $\lambda_{\min}\left(\mat{\hat{G}}^{(H)}(t)\right) \ge \frac{\lambda_0}{2}$ for $t'=0,\ldots,t$, where $\lambda_0$ is independent of $m$.
Similar to \cite{du2019gradient}, we can show  at the beginning 
\begin{align*}
	\lambda_{\min}\left(\mat{\hat{G}}^{(H)}(0)\right) \ge \frac{3}{4}\lambda_0.
\end{align*}

Now for the $t$-th global iteration, by matrix perturbation analysis, we know it is sufficient to show $\norm{\mat{\hat{G}}^{(H)}(t)-\mat{\hat{G}}^{(H)}(0)}_2 \le \frac{1}{4}\lambda_0$.
To do this, we show as long as $m$ is large enough, every weight matrix is close its initialization in a relative error sense.

\begin{lem}[Lemma on Initialization Norms for the whole network]
	\label{lem:init_norm_res_global}
	If $\sigma(\cdot)$ is $L-$Lipschitz and $m = \Omega\left(\frac{n}{\delta}\right)$, assuming $\norm{\mathcal{W}^{(h)}_0}_2\le c_{w,0}\sqrt{m}$ for $h\in[2,H]$ and $c_{w,0} \approx 2$ for Gaussian initialization. We have with probability at least $1-\delta$ over random initialization, for every $h\in[H]$ and $i \in [n]$, 
	\[ \frac{1}{c_{x,0}}\le \norm{\vect{x}_{i,0}^{(h)}}_2 \le c_{x,0} \]  for some universal constant $c_{x,0} > 1$
\end{lem}

\begin{proof}[Proof of Lemma~\ref{lem:init_norm_res_global}]
As the global model at initialization is the same with original ResNet in \cite{du2019gradient}, we can use the same proof in Lemma C.1 in \cite{du2019gradient}.
\end{proof}

The following lemma lower bounds $\mat{\hat{G}}^{(H)}(0)$'s least eigenvalue.
\begin{lem}[Least Eigenvalue at the Initialization]\label{lem:resnet_least_eigen_whole}
	If $m = \Omega\left(\frac{n^2\log(Hn/\delta)}{\lambda_0^2}\right)$, we have \begin{align*}
	\lambda_{\min}(\mat{\hat{G}}^{(H)}(0)) \ge \frac{3}{4}\lambda_0.
	\end{align*}
\end{lem}

\begin{proof}[Proof of Lemma~\ref{lem:resnet_least_eigen_whole}]
As the global model at initialization is the same with original ResNet in \cite{du2019gradient}, we can use the same proof in Lemma C.2 in \cite{du2019gradient}.
\end{proof}

\begin{lem}\label{lem:pertubation_of_neuron_res_sub}
	Suppose $\sigma(\cdot)$ is $L$-Lipschitz and for $h\in[H]$, $\norm{\mathcal{W}^{(h)}_0}_2 \le c_{w,0}\sqrt{m}$, $\norm{\vect{x}^{(h)}_0}_2 \le c_{x,0}$ and $\norm{\mat{W}^{(h)}_{v,l_t,t}-\mathcal{W}^{(h)}_0}_F \le \sqrt{m} R$ for some constant $c_{w,0},c_{x,0} > 0$ and $R\le c_{w,0}$ .
	Then we have \begin{align*}
	\norm{\vect{x}^{(h)}_{v,l_t,t}-\vect{x}^{(h)}_0}_2 \le \left(\sqrt{c_{\sigma}}L+\frac{c_{x,0}}{c_{w,0}}+\frac{c_{x,0}}{R}\right)e^{2c_{res}c_{w,0}L} R \triangleq c'_xR.
	\end{align*} 
\end{lem}

\begin{proof}[Proof of Lemma~\ref{lem:pertubation_of_neuron_res_sub}]
	We prove this lemma by induction.
	Our induction hypothesis is \begin{align*}
	\norm{\vect{x}^{(h)}_{v,l_t,t}-\vect{x}^{(h)}_{0}}_2 \le g(h)  ,
	\end{align*}	where \begin{align*}
	g(h) = \left[1+\frac{2c_{res}c_{w,0}L}{H}\right]g(h-1) + \frac{c_{res}Lc_{x,0}}{H}(c_{w,0}+R).
	\end{align*}
	For $h=1$, we have
	\begin{align*}
	\norm{\vect{x}^{(1)}_{v,l_t,t}-\vect{x}^{(1)}_0}_2&\le \sqrt{\frac{c_{\sigma}}{m}} \norm{\relu{\mat{W}^{(1)}_{v,l_t,t} \vect{x}} -\relu{\mathcal{W}^{(1)}_0 \vect{x}}}_2\\
&	\le \sqrt{\frac{c_{\sigma}}{m}}L\norm{\mat{W}^{(1)}_{v,l_t,t}-\mathcal{W}^{(1)}_0}_F \le \sqrt{c_{\sigma}}LR ,
	\end{align*}
	which implies $g(1)=\sqrt{c_{\sigma}}LR$, for $2\le h\le H$, we have
	\begin{align*}
&	\norm{\vect{x}^{(h)}_{v,l_t,t}-\vect{x}^{(h)}_0}_2 \le \frac{c_{res}}{H\sqrt{m}} \norm{\relu{\mat{W}^{(h)}_{v,l_t,t} \vect{x}^{(h-1)}_{v,l_t,t}}M^{(h)}_{v,t} -\relu{\mathcal{W}^{(h)}_0 \vect{x}^{(h-1)}_0}1}_2\\&+\norm{\vect{x}^{(h-1)}_{v,l_t,t}-\vect{x}^{(h-1)}_0}_2 \\
&   \le \frac{c_{res}}{H\sqrt{m}} \norm{[\relu{\mat{W}^{(h)}_{v,l_t,t} \vect{x}^{(h-1)}_{v,l_t,t}} -\relu{\mathcal{W}^{(h)}_0 \vect{x}^{(h-1)}_0}]M^{(h)}_{v,t} + \relu{\mathcal{W}^{(h)}_0 \vect{x}^{(h-1)}_0}(M^{(h)}_{v,t}-1)}_2\\&+\norm{\vect{x}^{(h-1)}_{v,l_t,t}-\vect{x}^{(h-1)}_0}_2 \\
&   \le \frac{c_{res}}{H\sqrt{m}} \norm{[\relu{\mat{W}^{(h)}_{v,l_t,t} \vect{x}^{(h-1)}_{v,l_t,t}} -\relu{\mathcal{W}^{(h)}_0 \vect{x}^{(h-1)}_0}]M^{(h)}_{v,t}}_2 + \frac{c_{res}}{H\sqrt{m}} \norm{\relu{\mathcal{W}^{(h)}_0 \vect{x}^{(h-1)}_0}(M^{(h)}_{v,t}-1)}_2\\&+\norm{\vect{x}^{(h-1)}_{v,l_t,t}-\vect{x}^{(h-1)}_0}_2 \\
&   \le \frac{c_{res}}{H\sqrt{m}} \norm{[\relu{\mat{W}^{(h)}_{v,l_t,t} \vect{x}^{(h-1)}_{v,l_t,t}} -\relu{\mathcal{W}^{(h)}_0 \vect{x}^{(h-1)}_0}]}_2\norm{M^{(h)}_{v,t}}_2 + \frac{c_{res}}{H\sqrt{m}} \norm{\relu{\mathcal{W}^{(h)}_0 \vect{x}^{(h-1)}_0}}_2\norm{(M^{(h)}_{v,t}-1)}_2\\&+\norm{\vect{x}^{(h-1)}_{v,l_t,t}-\vect{x}^{(h-1)}_0}_2 \\
&   \le \frac{c_{res}}{H\sqrt{m}} \norm{[\relu{\mat{W}^{(h)}_{v,l_t,t} \vect{x}^{(h-1)}_{v,l_t,t}} -\relu{\mat{W}^{(h)}_{v,l_t,t} \vect{x}^{(h-1)}_0} + \relu{\mat{W}^{(h)}_{v,l_t,t} \vect{x}^{(h-1)}_{v,l_t,t}}-\relu{\mathcal{W}^{(h)}_0 \vect{x}^{(h-1)}_0}]}_2 \\&+ \frac{c_{res}L}{H\sqrt{m}} \norm{\mathcal{W}^{(h)}_0 \vect{x}^{(h-1)}_0}+\norm{\vect{x}^{(h-1)}_{v,l_t,t}-\vect{x}^{(h-1)}_0}_2 \\
& 	\le \frac{c_{res}}{H\sqrt{m}} \norm{\relu{\mat{W}^{(h)}_{v,l_t,t} \vect{x}^{(h-1)}_{v,l_t,t}} -\relu{\mat{W}^{(h)}_{v,l_t,t} \vect{x}^{(h-1)}_0}}_2\\ 
	& + \frac{c_{res}}{H\sqrt{m}} \norm{\relu{\mat{W}^{(h)}_{v,l_t,t} \vect{x}^{(h-1)}_0} -\relu{\mathcal{W}^{(h)}_0 \vect{x}^{(h-1)}_0}}_2\\ &+\norm{\vect{x}^{(h-1)}_{v,l_t,t}-\vect{x}^{(h-1)}_0}_2+ \frac{c_{res}Lc_{w,0}c_{x,0}}{H}\\
&	\le  \frac{c_{res}L}{H\sqrt{m}}\left(
	\norm{\mathcal{W}^{(h)}_0}_2 + \norm{\mat{W}^{(h)}_{v,l_t,t}-\mathcal{W}^{(h)}_0}_F
	\right) \cdot \norm{\vect{x}^{(h-1)}_{v,l_t,t}-\vect{x}^{(h-1)}_0}_2 \\
	& + \frac{c_{res}L}{H\sqrt{m}}\norm{\mat{W}^{(h)}_{v,l_t,t}-\mathcal{W}^{(h)}_0}_F \norm{\vect{x}^{(h-1)}_0}_2 +\norm{\vect{x}^{(h-1)}_{v,l_t,t}-\vect{x}^{(h-1)}_0}_2+ \frac{c_{res}Lc_{w,0}c_{x,0}}{H}\\
&	\le \left[1+\frac{c_{res}L}{H\sqrt{m}}\left(c_{w,0}\sqrt{m}+R\sqrt{m}\right)\right]g(h-1) + \frac{c_{res}L}{H\sqrt{m}} \sqrt{m} R c_{x,0}+ \frac{c_{res}Lc_{w,0}c_{x,0}}{H}\\
&	\le \left(1+\frac{2c_{res}c_{w,0}L}{H}\right)g(h-1)+\frac{c_{res}}{H} Lc_{x,0}R + \frac{c_{res}Lc_{w,0}c_{x,0}}{H}. 
	\end{align*}
	Lastly, simple calculations show $g(h) \le \left(\sqrt{c_{\sigma}}L+\frac{c_{x,0}}{c_{w,0}}+\frac{c_{x,0}}{R}\right)e^{2c_{res}c_{w,0}L} R$. 
	
\end{proof}

\begin{lem}\label{lem:pertubation_of_neuron_res_whole}
	Suppose $\sigma(\cdot)$ is $L$-Lipschitz and for $h\in[H]$, $\norm{\mathcal{W}^{(h)}_0}_2 \le c_{w,0}\sqrt{m}$, $\norm{\vect{x}^{(h)}_0}_2 \le c_{x,0}$ and $\norm{\mathcal{W}^{(h)}_{t}-\mathcal{W}^{(h)}_0}_F \le \sqrt{m} R$ for some constant $c_{w,0},c_{x,0} > 0$ and $R\le c_{w,0}$ .
	Then we have \begin{align*}
	\norm{\vect{x}^{(h)}_{t}-\vect{x}^{(h)}_0}_2 \le \left(\sqrt{c_{\sigma}}L+\frac{c_{x,0}}{c_{w,0}}\right)e^{2c_{res}c_{w,0}L} R \triangleq c_xR.
	\end{align*} 
\end{lem}

\begin{proof}[Proof of Lemma~\ref{lem:pertubation_of_neuron_res_whole}]
The proof is exactly the same with proof of C.3 in \cite{du2019gradient}
\end{proof}

Next, we characterize how the perturbation on the weight matrices affect $\mat{\hat{G}}^{(H)}$.
\begin{lem}\label{lem:close_to_init_small_perturbation_res_smooth}
	Suppose $\sigma(\cdot)$ is differentiable, $L-$Lipschitz and $\beta-$smooth. Suppose $\sigma(\cdot)$ is $L-$Lipschitz and $\beta-$smooth. Suppose for $h\in[H]$, $\norm{\mathcal{W}^{(h)}_0}_2\le c_{w,0}\sqrt{m}$, $\norm{\vect{a}_0}_2\le a_{2,0}\sqrt{m}$, $\norm{\vect{a}_0}_4\le a_{4,0}m^{1/4}$ , $\frac{1}{c_{x,0}}\le\norm{\vect{x}^{(h)}_0}_2 \le c_{x,0}$, if $\norm{\mathcal{W}^{(h)}_r-\mathcal{W}^{(h)}_0}_F, \norm{\vect{a}_r-\vect{a}_0}_2\le \sqrt{m}R$ where $R \le c \lambda_0H^2n^{-1}$ and $R\le c$ for some small constant $c$ , we have  \begin{align*}
	\norm{\mat{\hat{G}}^{(H)}(t) - \mat{\hat{G}}^{(H)}(0)}_2 \le \frac{\lambda_0}{2}.
	\end{align*}
\end{lem}

\begin{proof}[Proof of Lemma~\ref{lem:close_to_init_small_perturbation_res_smooth}]
	Similar to C.4 in \cite{du2019gradient} Because Frobenius-norm of a matrix is bigger than the operator norm, it is sufficient to bound $\norm{\mat{\hat{G}}^{(H)}(t) - \mat{\hat{G}}^{(H)}(0)}_F$. 
	For simplicity define $z_{i,q}(t) = \mathcal{W}_{t,q}^{{(H)}\top} \vect{x}_{i,t}^{(H-1)}$, we have
	\begin{align*}
	&\abs{\mat{\hat{G}}_{i,j}^{(H)}(t)-\mat{\hat{G}}_{i,j}^{(H)}(0)} \\
	= & \frac{c_{res}^2}{H^2m} \big{|}\vect{x}_{i,t}^{(H-1)\top} \vect{x}_{j,t}^{(H-1)}
	  \sum_{q=1}^{m}a_q(t)^2\sigma'\left(z_{i,q}(t)\right)\sigma'\left(z_{j,q}(t)\right)
	\\
	&-\vect{x}_{i,0}^{(H-1)\top} \vect{x}_{j,0}^{(H-1)}
	 \sum_{q=1}^{m}a_q(0)^2\sigma'\left(z_{i,q}(0)\right)\sigma'\left(z_{j,q}(0)\right)
	\big{|} \\
\le & \frac{c_{res}^2}{H^2}L^2a_{2,0}^2\abs{\vect{x}_{i,t}^{(H-1)\top} \vect{x}_{j,t}^{(H-1)} - \vect{x}_{i,0}^{(H-1)\top} \vect{x}_{j,0}^{(H-1)}} \\
	&+ \frac{c_{res}^2}{H^2}\frac{c_{x,0}^2}{m}\abs{\sum_{q=1}^{m}a_q(0)^2\left(	\sigma'\left(z_{i,q}(t)\right)\sigma'\left(z_{j,q}(t)\right)
		- 
		\sigma'\left(z_{i,q}(0)\right)\sigma'\left(z_{j,q}(0)\right)\right)} \\
&+\frac{c_{res}^2}{H^2m} \abs{\vect{x}_{i,t}^{(H-1)\top} \vect{x}_{j,t}^{(H-1)}}  \abs{\sum_{q=1}^{m} \left(a_q(t)^2-a_q(0)^2\right)
		\sigma'\left(z_{i,q}(t)\right)\sigma'\left(z_{j,q}(t)\right)
	}\\
	\triangleq& \frac{c_{res}^2}{H^2}(I_1^{i,j} + I_2^{i,j} + I_3^{i,j}).
	\end{align*}
	
	For $I_1^{i,j}$, using Lemma~\ref{lem:pertubation_of_neuron_res_whole}, we have \begin{align*}
	I_1^{i,j} = &L^2a_{2,0}^2\abs{\vect{x}_{i,t}^{(H-1)\top} \vect{x}_{j,t}^{(H-1)} - \vect{x}_{i,0}^{(H-1)\top} \vect{x}_{j,0}^{(H-1)}}  \\
	\le & L^2a_{2,0}^2\abs{
		(\vect{x}_{i,t}^{(H-1)}-\vect{x}_{i,0}^{(H-1)}\top) \vect{x}_{j,t}^{(H-1)}} + L^2a_{2,0}^2\abs{
		\vect{x}_{i,0}^{(H-1)\top}(\vect{x}_{i,t}^{(H-1)}-\vect{x}_{i,0}^{(H-1)})}  \\
	\le & c_{x}L^2 a_{2,0}^2R \cdot (c_{x,0} + c_{x} R) + c_{x,0} c_x L^2a_{2,0}^2R \\
	\le &3 c_{x,0} c_x L^2a_{2,0}^2 R,
	\end{align*}
	
	Same with C.4 \cite{du2019gradient}, to bound $I_{2}^{i,j}$, we have 
	\begin{align*}
	I_2^{i,j} =&c_{x,0}^2 \frac{1}{m} \abs{\sum_{q=1}^{m}
		a_q(0)^2\sigma'\left(z_{i,q}(t)\right)\sigma'\left(z_{j,q}(t)\right)
		- 
		a_q(0)^2\sigma'\left(z_{i,q}(0)\right)\sigma'\left(z_{j,q}(0)\right)
	}\\
	\le & \frac{\beta La_{4,0}^2 c_{x,0}^2}{\sqrt{m}}\left(\sqrt{\sum_{q=1}^{m}\abs{z_{i,q}(t)-z_{i,q}(0)}^2}+\sqrt{\sum_{q=1}^{m}\abs{z_{j,q}(t)-z_{j,q}(0)}^2}\right) .
	\end{align*}
	Using the same proof for Lemma~\ref{lem:pertubation_of_neuron_res_whole}, it is easy to see \begin{align*}
	\sum_{q=1}^{m}\abs{z_{i,q}(t)-z_{i,q}(0)}^2 \le \left(2c_xc_{w,0}+c_{x,0}\right)^2L^2mR^2 .
	\end{align*}
	Thus
	\begin{align*}
	I_2^{i,j}\le 2\beta c_{x,0}^2 \left(2c_xc_{w,0}+c_{x,0}\right)L^2 R .
	\end{align*}
	The bound of $I_3^{i,j}$ is the same to that $I_3^{i,j}$ in \cite{du2019gradient} C.4,
		\begin{align*}
	I_3^{i,j}
	&\le 12L^2c_{x,0}^2 a_{2,0}R.
	\end{align*}
	Therefore we can bound the perturbation\begin{align*}
	\norm{\mat{\hat{G}}^{(H)}(t) - \mat{\hat{G}}^{(H)}(0)}_F=&\sqrt{\sum_{(i,j)}^{{n,n}} \abs{\mat{\hat{G}}_{i,j}^{(H)}(t)-\mat{\hat{G}}_{i,j}^{(H)}(0)}^2} \\
	&\le \frac{c_{res}^2}{H^2}\sqrt{n^2(3 c_{x,0} c_xL^2a_{2,0}^2 R+2\beta c_{x,0}^2 \left(2c_xc_{w,0}+c_{x,0}\right)L^2 R+12L^2c_{x,0}^2 a_{2,0}R}) \\
	&= \frac{c_{res}^2}{H^2}n(3 c_{x,0} c_xL^2a_{2,0}^2 R+2\beta c_{x,0}^2 \left(2c_xc_{w,0}+c_{x,0}\right)L^2 R+12L^2c_{x,0}^2 a_{2,0}R) \\. \\
	\end{align*}
	Plugging in the bound on $R$, we have the desired result.
\end{proof}

Now we prove theorem \ref{thm:resist_gd} by induction, assume the condition \ref{cond:linear_converge_resist}, we want to bound the change of weight to satisfy lemma \ref{lem:close_to_init_small_perturbation_res_smooth} and then we want to show 
$2 \left(\vect{y}-\vect{\hat{u}}(t)\right)^\top(\vect{I'}_1(t)-\vect{I}_1(t))$,$-2\left(\vect{y}-\vect{\hat{u}}(t)\right)^\top\vect{I}_2(t)$ and  $\norm{\vect{\hat{u}}(t+1)-\vect{\hat{u}}(t)}_2^2$  are proportional to $\eta^2 \norm{\vect{y}-\vect{u}(t)}_2^2$
so if we set $\eta$ sufficiently small, this term is smaller than $\eta \lambda_{\min}\left(\mat{\hat{G}}^{(H)}(t)\right)\norm{\vect{y}-\vect{u}(t)}_2^2$ and thus the loss function decreases with a linear rate.
\begin{lem}\label{lem:dist_from_init_resnet_sub}
	If Condition~\ref{cond:linear_converge_resist} holds for $t'=0,\ldots,t-1$, we have for any $1  \le v \le S$, $0 \le l_t \le \ell$
	\begin{align*}
	&\norm{\mat{W}^{(h)}_{v,l_t,t}-\mathcal{W}^{(h)}_0}_F, \norm{\vect{a}_{v,l_t,t}-\vect{a}_0}_2\le  R'\sqrt{m},\\
	&\norm{\mat{W}^{(h)}_{v,l_t,t}-\mat{W}^{(h)}_{v,l_t-1,t}}_F,\norm{\vect{a}_{v,l_t,t}-\vect{a}_{v,l_t-1,t}}_2\le \eta Q'(l_t-1,t),
	\end{align*} where $R'=\frac{16 c_{res}c_{x,0}a_{2,0}Le^{2c_{res}c_{w,0}L} \sqrt{n} \norm{\vect{y}-\vect{u}(0)}_2}{H\lambda_0\sqrt{m}} <c$ for some small constant $c$ , \\$ Q'(l_t,t)= 4c_{res}c_{x,0}a_{2,0}Le^{2c_{res}c_{w,0}L}\sqrt{n} \norm{\vect{y}-\vect{u}_{t,l_t}}_2/H$ and \\$ Q'(t)= 4c_{res}c_{x,0}a_{2,0}Le^{2c_{res}c_{w,0}L}\sqrt{n} \norm{\vect{y}-\vect{\hat{u}}_{t}}_2/H$.
\end{lem}

\begin{proof}[Proof of Lemma~\ref{lem:dist_from_init_resnet_sub}]
	We will prove this corollary by induction. The induction hypothesis is
		\begin{align*}
	\norm{\mat{W}^{(h)}_{v,l_t,t}-\mathcal{W}^{(h)}_0}_F &\le  R'\sqrt{m}\\
	\norm{\vect{a}_{v,l_t,t}-\vect{a}_0}_2 &\le  R'\sqrt{m}.
	\end{align*}
	First we want to prove it holds for $t'=0$ and $0 \le l_t \le \ell$. 
	
	We prove it by induction w.r.t $l_t$:
	It is easy to see that it holds for $t'=0$ and $l_t'=0$.
	Suppose it holds for $0 \le l_t' \le l_t$, we want to prove it holds for $l_t'=l_t+1$
	Following C.5 in \cite{du2019gradient}, note $\norm{\mat{J}_{i,v,l_t,t}^{(k)}}_2 \le L$.
	We have
	\begin{align*} 
	&\norm{\mat{W}^{(h)}_{v,l_t+1,t}-\mat{W}^{(h)}_{v,l_t,t}}_F\\
	\le&  \eta \frac{c_{res}}{H\sqrt{m}} \norm{\vect{a_{v,l_t,t}}}_2 \sum_{i=1}^{n}\abs{y_i-u_{i,v,l_t,t}}\norm{\vect{x}^{(h-1)}_{i,v,l_t,t}}_2 \prod_{k=h+1}^H \norm{\mat{I}+\frac{c_{res}}{H\sqrt{m}}\mat{J}_{i,v,l_t,t}^{(k)}\mat{W}^{(k)}_{v,l_t,t}M_{v,t}^{(k)} }_2 \norm{\mat{J}_{i,v,l_t,t}^{(k)}}_2 \norm{M_{v,t}^{(h)}}_2 \\
	\le &\eta \frac{Lc_{res}}{H\sqrt{m}} \norm{\vect{a_{v,l_t,t}}}_2 \sum_{i=1}^{n}\abs{y_i-u_{i,v,l_t,t}}\norm{\vect{x}^{(h-1)}_{i,v,l_t,t}}_2 \prod_{k=h+1}^H \norm{\mat{I}+\frac{c_{res}}{H\sqrt{m}}\mat{J}_{i,v,l_t,t}^{(k)}\mat{W}^{(k)}_{v,l_t,t}M_{v,t}^{(k)} }_2  \\
	\end{align*}
	
	Further
	\begin{align*} 
	&\prod_{k=h+1}^H \norm{\mat{I}+\frac{c_{res}}{H\sqrt{m}}\mat{J}_{i,v,l_t,t}^{(k)}\mat{W}^{(k)}_{v,l_t,t} M_{v,t}^{(k)}}_2 \\
	\le & \prod_{k=h+1}^H \norm{\mat{I}}_2+\norm{\frac{c_{res}}{H\sqrt{m}}\mat{J}_{i,v,l_t,t}^{(k)}\mat{W}^{(k)}_{v,l_t,t} M_{v,t}^{(k)}}_2 \\
	\le & \prod_{k=h+1}^H \norm{\mat{I}}_2+\frac{c_{res}}{H\sqrt{m}}\norm{\mat{J}_{i,v,l_t,t}^{(k)}}_2\norm{\mat{W}^{(k)}_{v,l_t,t}}_2\norm{ M_{v,t}^{(k)}}_2 \\
	\le & \prod_{k=h+1}^H \norm{\mat{I}}_2+\frac{c_{res}L}{H\sqrt{m}}(\norm{\mathcal{W}^{(k)}_0}_F+\norm{\mat{W}^{(k)}_{v,l_t,t}-\mathcal{W}^{(k)}_0}_F) \\
	\le & \prod_{k=h+1}^H 1+\frac{c_{res}L}{H}(c_{w,0}+R') \\
	\le & \prod_{k=h+1}^H 1+\frac{c_{res}L}{H}2c_{w,0} \\
	\le & e^{2c_{res}x_{w,0}L}
	\end{align*}

	Thus
	\begin{align*} 
	&\norm{\mat{W}^{(h)}_{v,l_t+1,t}-\mat{W}^{(h)}_{v,l_t,t}}_F\\
	\le &\eta \frac{Lc_{res}}{H\sqrt{m}} \norm{\vect{a}_{v,l_t,t}}_2 \sum_{i=1}^{n}\abs{y_i-u_i(s)}\norm{\vect{x}^{(h-1)}_{i,v,l_t,t}}_2 \prod_{k=h+1}^H \norm{\mat{I}+\frac{c_{res}}{H\sqrt{m}}\mat{J}_{i,v,l_t,t}^{(k)}\mat{W}^{(k)}_{v,l_t,t} M_{v,t}^{(k)}}_2\\
	\le &\eta \frac{Lc_{res}}{H\sqrt{m}} \norm{\vect{a}_{v,l_t,t}}_2 \sum_{i=1}^{n}\abs{y_i-u_i(s)}\norm{\vect{x}^{(h-1)}_{i,v,l_t,t}}_2 e^{2c_{res}x_{w,0}L}\\
	\le &\eta c_{res}(c_{x,0}+c_{x}R')La_{2,0}e^{2c_{res}c_{w,0}L}\sqrt{n} \norm{\vect{y}-\vect{u}_{l_t,t}}_2/H\\
	\le &3\eta c_{res}c_{x,0}La_{2,0}e^{2c_{res}c_{w,0}L}\sqrt{n} \norm{\vect{y}-\vect{u}_{l_t,t}}_2/H\\
	\le & \eta Q'(l_t,t)\\
	\le& (1-\frac{\eta \lambda_0}{2})^{s/2}\frac{1}{4}\eta \lambda_0 R'\sqrt{m}
	\end{align*}
		Similarly, we have \begin{align*}
	\norm{\vect{a}_{v,l_t+1,t}-\vect{a}_{v,l_t,t}}_2 \le& 3\eta c_{x,0} \sum_{i=1}^{n}\abs{y_i-u_{l_t,t}}\\
	\le& \eta Q'(l_t,t)\\
	\le& (1-\frac{\eta \lambda_0}{2})^{l_t/2}\frac{1}{4}\eta \lambda_0 R'\sqrt{m}.
	\end{align*}
	
	Thus
	\begin{align*}
	&\norm{\mat{W}^{(h)}_{v,l_t+1,t}-\mathcal{W}^{(h)}_0}_F\\
	\le& \norm{\mat{W}^{(h)}_{v,l_t+1,t}-\mathcal{W}^{(h)}_{v,l_t,t}}_F+\norm{\mat{W}^{(h)}_{v,l_t,t}-\mathcal{W}^{(h)}_0}_F\\
	\le&\sum_{l_t'=0}^{l_t} \eta (1-\frac{\eta \lambda_0}{2})^{l_t'/2}\frac{1}{4}\eta \lambda_0 R'\sqrt{m}.\\
	\end{align*}
		Similarly,
	\begin{align*}
	&\norm{\vect{a}_{v,l_t+1,t}-\vect{a}_0}_2\\
	\le&\sum_{l_t'=0}^{l_t} \eta (1-\frac{\eta \lambda_0}{2})^{l_t'/2}\frac{1}{4}\eta \lambda_0 R'\sqrt{m}.\\
	\end{align*}
	
	Now suppose the hypothesis hold for t'=0,1..,t and for $0 \le l_t \le \ell$. We want to prove for $t'=t+1$, the hypothesis holds.
	By Lemma ~\ref{lem:weight_sub_whole}, we know $\norm{\mathcal{W}^{(h)}_{t}-\mathcal{W}^{(h)}_0}_F \le \sqrt{m} R'$ 
	Thus, $\norm{\mat{W}^{(h)}_{v,l_t=0,t+1}-\mathcal{W}^{(h)}_0}_F \le \sqrt{m} R'$
	Thus, by using the same induction on $l_t$ above, we can prove the hypothesis for $t+1$.
	
\end{proof}

\begin{lem}\label{lem:weight_sub_whole}
Assume 
\begin{align*}
&\norm{\mat{W}^{(h)}_{v,l_t,t}-\mathcal{W}^{(h)}_0}_F, \norm{\vect{a}_{v,l_t,t}-\vect{a}_0}_2 \le \sqrt{m} R' \\
\end{align*} 
We have
\begin{align*}
\norm{\mathcal{W}^{(h)}_{t}-\mathcal{W}^{(h)}_0}_F,\norm{\vect{a}_{t}-\vect{a}_0}_2 \le \sqrt{m} R'
\end{align*} 

\end{lem}

\begin{proof}[Proof of Lemma~\ref{lem:weight_sub_whole}]
\begin{align*}
\norm{\mathcal{W}^{(h)}_{t}-\mathcal{W}^{(h)}_0}_F &= \norm{\frac{\sum_{v=1}^{S} \mat{W}^{(h)}_{v,\ell,t}M^{(h)}_{v,t}}{\sum_{v=1}^{S}M^{(h)}_{v,t}} - \mathcal{W}^{(h)}_0}_F \\
& \le \frac{\sum_{v: M^{(h)}_{v,t}=1} \norm{\mat{W}^{(h)}_{v,\ell,t}-\mathcal{W}^{(h)}_0}_F}{\sum_{v=1}^{S}M^{(h)}_{v,t}} \\
& \le \sqrt{m} R'
\end{align*} 
Similarly,
\begin{align*}
\norm{\vect{a}_{t}-\vect{a}_0}_2 & \le \frac{\sum_{v=1}^{S} \norm{\vect{a}_{v,l_t,t}-\vect{a}_0}_2}{S} \\
& \le \sqrt{m} R'
\end{align*} 
\end{proof}

\begin{lem}\label{lem:dist_from_wholenetwork}
If Condition~\ref{cond:linear_converge_resist} holds for $t'=0,\ldots,t-1$ and $\eta\le c\lambda_0H^2n^{-2} \ell^{-2} S^{-1}$ for some small constant $c$, we have
$\norm{\vect{I'}^{i}_1(t)-\vect{I}^{i}_1(t)}_2 \le  C_{I_1}^{*} \eta^2 \norm{y_i-\hat{u}_{i,t-1}}_2$ where $C_{I_1}^{*}$ is a constant and thus $\norm{\vect{I'}_1(t)-\vect{I}_1(t)}_2 \le \frac{1}{16}\eta \lambda_0 \norm{\vect{y}-\vect{\hat{u}}(k)}_2$.
\end{lem}

\begin{proof}[Proof of Lemma~\ref{lem:dist_from_wholenetwork}]
\begin{align*}
\norm{\vect{I'}^{i}_1(t)-\vect{I}^{i}_1(t)}_2 & =\norm{\langle \eta \ell L'(\params(t))-(\theta(t+1)-\theta(t)), \hat{u}'_i\left(\params(t)\right) \rangle}_2 \\
	& \le \sum_{h=1}^{H} \norm{\eta \frac{\sum_{v=1}^{S} \sum_{l_t=1}^{\ell} \frac{\partial L}{\partial \mat{W}^{(h)}_{v,l_t,t}}}{\sum_{v=1}^{S}M^{(h)}_{v,t}}-\eta \ell \frac{\partial L}{\partial\mathcal{W}_{t-1}^{(h)}}}_F \norm{\hat{u}'_i\left(\params(t\right)}_2 \\
	& + \norm{\eta \frac{\sum_{v=1}^{S} \sum_{l_t=1}^{\ell} \frac{\partial L}{\partial \vect{a}_{v,l_t,t}}}{S}-\eta \ell \frac{\partial L}{\partial\vect{a}_t}}_2 \norm{\hat{u}'_i\left(\params(t\right)}_2
\end{align*}
	Let $M_{t,h}=\sum_{v=1}^{S}M^{(h)}_{v,t}$
	\begin{align*}
	&\norm{\eta \frac{\sum_{v=1}^{S} \sum_{l_t=1}^{\ell} \frac{\partial L}{\partial \mat{W}^{(h)}_{v,l_t,t}}}{\sum_{v=1}^{S}M^{(h)}_{v,t}}-\eta \ell \frac{\partial L}{\partial\mathcal{W}_{t-1}^{(h)}}}_F \\
	\le &   \eta 1/M_{t,h} \sum_{v=1}^{S} \sum_{l_t=1}^{\ell} \norm{\frac{\partial L}{\partial \mat{W}^{(h)}_{v,l_t,t}}-\frac{\partial L}{\partial\mathcal{W}_{t-1}^{(h)}}}_F \\
	\le & \eta 1/M_{t,h} \sum_{v=1}^{S} \sum_{l_t=1}^{\ell} \big{||} \frac{c_{res}}{H\sqrt{m}}
\sum_{i=1}^{n}(y_i-u_{i,v,l_t,t})\vect{x}_{i,v,l_t,t}^{(h-1)} \cdot
\left[\vect{a}_{v,l_t,t}^\top \prod_{l=h+1}^{H}\left(\mat{I}+\frac{c_{res}}{H\sqrt{m}}\mat{J}_{i,v,l_t,t}^{(l)}\mat{W}_{v,l_t,t}^{(l)} M^{(l)}_{v,t} \right) \mat{J}_{i,v,l_t,t}^{(h)}  M^{(h)}_{v,t}\right] \\
& - \frac{c_{res}}{H\sqrt{m}}
\sum_{i=1}^{n}(y_i-\hat{u}_{i,t-1})\vect{x}_{i,t-1}^{(h-1)} \cdot
\left[\vect{a}_{t-1}^\top \prod_{l=h+1}^{H}\left(\mat{I}+\frac{c_{res}}{H\sqrt{m}}\mat{J}_{i,t-1}^{(l)}\mathcal{W}_{t-1}^{(l)}  \right) \mat{J}_{i,t-1}^{(h)}  \right] \big{||}_F \\
\le & \eta 1/M_{t,h} \sum_{v=1}^{S} \sum_{l_t=1}^{\ell} \frac{c_{res}}{H\sqrt{m}} \sum_{i=1}^{n} \big{||} 
(y_i-u_{i,v,l_t,t})\vect{x}_{i,v,l_t,t}^{(h-1)} \cdot
\left[\vect{a}_{v,l_t,t}^\top \prod_{l=h+1}^{H}\left(\mat{I}+\frac{c_{res}}{H\sqrt{m}}\mat{J}_{i,v,l_t,t}^{(l)}\mat{W}_{v,l_t,t}^{(l)} M^{(l)}_{v,t} \right) \mat{J}_{i,v,l_t,t}^{(h)}  M^{(h)}_{v,t}\right] \\
& - (y_i-\hat{u}_{i,t-1})\vect{x}_{i,t-1}^{(h-1)} \cdot
\left[\vect{a}_{t-1}^\top \prod_{l=h+1}^{H}\left(\mat{I}+\frac{c_{res}}{H\sqrt{m}}\mat{J}_{i,t-1}^{(l)}\mathcal{W}_{t-1}^{(l)}  \right) \mat{J}_{i,t-1}^{(h)}  \right] \big{||}_F \\
	\end{align*}
	Through standard calculations, we have
	\begin{align*}
	\norm{	\mathcal{W}_{t-1}^{(l)}-\mat{W}_{v,l_t,t}^{(l)}}_F \le &\eta \ell Q'(0,t),\\
		\norm{	\vect{a}_{t-1}-\vect{a}_{v,l_t,t}}_F \le &\eta \ell Q'(0,t),\\
	\norm{	\vect{x}_{i,t-1}^{(h-1)}-\vect{x}_{i,v,l_t,t}^{(h-1)}}_F \le &\eta \ell c'_x \frac{Q'(0,t)}{\sqrt{m}},\\
	\norm{	\mat{J}_{i,t-1}^{(l)}-\mat{J}_{i,v,l_t,t}^{(l)}}_F \le &2 \ell \left(c_{x,0}+c_{w,0}c'_x\right)\eta \beta   Q'(0,t),
	\end{align*}
	where $c'_x\triangleq\left(\sqrt{c_{\sigma}}L+\frac{c_{x,0}}{c_{w,0}}+\frac{c_{x,0}}{R}\right)e^{2c_{res}c_{w,0}L}$.
As we know $\norm{y_i-u_{i,v,l_t,t}} \le \norm{y_i-\hat{u}_{i,t-1}} $, suppose $\norm{u_{i,v,l_t,t}-\hat{u}_{i,t-1}} \le C_u $
	
	According to Lemma G.1 in \cite{du2019gradient}, we have
	\begin{align*}
	& \eta 1/M_{t,h} \sum_{v=1}^{S} \sum_{l_t=1}^{\ell} \frac{c_{res}}{H\sqrt{m}} \sum_{i=1}^{n} \big{||} 
(y_i-u_{i,v,l_t,t})\vect{x}_{i,v,l_t,t}^{(h-1)} \cdot
\left[\vect{a}_{v,l_t,t}^\top \prod_{l=h+1}^{H}\left(\mat{I}+\frac{c_{res}}{H\sqrt{m}}\mat{J}_{i,v,l_t,t}^{(l)}\mat{W}_{v,l_t,t}^{(l)} M^{(l)}_{v,t} \right) \mat{J}_{i,v,l_t,t}^{(h)}  M^{(h)}_{v,t}\right] \\
& - (y_i-\hat{u}_{i,t-1})\vect{x}_{i,t-1}^{(h-1)} \cdot
\left[\vect{a}_{t-1}^\top \prod_{l=h+1}^{H}\left(\mat{I}+\frac{c_{res}}{H\sqrt{m}}\mat{J}_{i,t-1}^{(l)}\mathcal{W}_{t-1}^{(l)}  \right) \mat{J}_{i,t-1}^{(h)}  \right] \big{||}_F \\
	\le & \eta 1/M_{t,h} S \ell n 
	\frac{4}{H}c_{res}c_{x,0}La_{2,0}e^{2Lc_{w,0}}(C_u \\& +\eta \ell\frac{Q'(0,t)}{\sqrt{m}}\left(\frac{c_x}{c_{x,0}}+\frac{2}{L}\left(c_{x,0}+c_{w,0}c_x\right)\beta \sqrt{m}+4c_{w,0}\left(c_{x,0}+c_{w,0}c_x\right)\beta+L+1 \right))\norm{y_i-\hat{u}_{i,t-1}}_2\\
	\end{align*}
On the other hand,
\begin{align*}
    \norm{\eta \frac{\sum_{v=1}^{S} \sum_{l_t=1}^{\ell} \frac{\partial L}{\partial \vect{a}_{v,l_t,t}}}{S}-\eta \ell \frac{\partial L}{\partial\vect{a}_t}}_2 & \le   \eta 1/S \sum_{v=1}^{S} \sum_{l_t=1}^{\ell} \norm{\frac{\partial L}{\partial \vect{a}_{v,l_t,t}}-\frac{\partial L}{\partial\vect{a}_t}}_2 \\
    & \le \eta 1/S \sum_{v=1}^{S} \sum_{l_t=1}^{\ell} \sum_{i=1}^{n} \norm{(y_i-u_{i,v,l_t,t})\vect{x}_{i,v,l_t,t}^{(H)}-(y_i-\hat{u}_{i,t})\vect{x}_{i,t-1}^{(H)}}_2 \\
    & \le \eta \ell n (C_{u}+\eta \ell c'_x \frac{Q'(0,t)}{\sqrt{m}}) \norm{y_i-\hat{u}_{i,t-1}}_2
\end{align*}
Also,
\begin{align*}
\norm{\hat{u}'_i\left(\params(t)\right)}_2 & \le \frac{c_{res}}{H\sqrt{m}}\sum_{h=1}^H \norm{\frac{\partial \hat{u}_i\left(\params(t)\right)}{\partial \mathcal{W}^{(h)}_{t-1}}}_2 \\
& = \frac{c_{res}}{H\sqrt{m}}\sum_{h=1}^H \norm{\vect{x}_{i,t-1}^{(h-1)} \cdot
\left[\vect{a}_{t-1}^\top \prod_{l=h+1}^{H}\left(\mat{I}+\frac{c_{res}}{H\sqrt{m}}\mat{J}_{i,t-1}^{(l)}\mathcal{W}_{t-1}^{(l)}  \right) \mat{J}_{i,t-1}^{(h)}  \right]}_2 \\
& \le  \frac{c_{res}}{H\sqrt{m}}\sum_{h=1}^H \norm{\vect{x}_{i,t-1}^{(h-1)}}_2 \norm{\vect{a}_{t-1}}_2 \norm{\prod_{l=h+1}^{H}\left(\mat{I}+\frac{c_{res}}{H\sqrt{m}}\mat{J}_{i,t-1}^{(l)}\mathcal{W}_{t-1}^{(l)}  \right)}_2 \norm{\mat{J}_{i,t-1}^{(h)}}_2 \\
& \le \frac{c_{res}}{H}  H 2 c_{x,0} a_{2,0} L e^{2c_{res}x_{w,0}L} \\
& = 2c_{res}c_{x,0} a_{2,0} L e^{2c_{res}x_{w,0}L}
\end{align*}

Thus, combine all above and also according to Lemma \ref{lem:dist_u}
\begin{align*}
	\norm{\vect{I'}^{i}_1(t)-\vect{I}^{i}_1(t)}_2 & \le \sum_{h=1}^{H} \norm{\eta \frac{\sum_{v=1}^{S} \sum_{l_t=1}^{\ell} \frac{\partial L}{\partial \mat{W}^{(h)}_{v,l_t,t}}}{\sum_{v=1}^{S}M^{(h)}_{v,t}}-\eta \ell \frac{\partial L}{\partial\mathcal{W}_{t-1}^{(h)}}}_2 \norm{\hat{u}'_i\left(\params(t\right)}_2 \\
	& + \norm{\eta \frac{\sum_{v=1}^{S} \sum_{l_t=1}^{\ell} \frac{\partial L}{\partial \vect{a}_{v,l_t,t}}}{S}-\eta \ell \frac{\partial L}{\partial\vect{a}_t}}_2 \norm{\hat{u}'_i\left(\params(t\right)}_2 \\
	& \le C_{I_1}^{*} \eta^2 \norm{y_i-\hat{u}_{i,t-1}} \text{where $C_{I_1}^{*}$ is a constant}
\end{align*}
Using the bound on $\eta$ and following \cite{du2019gradient} $\norm{\vect{y}-\vect{\hat{u}}}_2=O(\sqrt{n})$,
\begin{align*}
\norm{\vect{I'}_1(t)-\vect{I}_1(t)} \le \frac{1}{16}\eta \lambda_0 \norm{\vect{y}-\vect{\hat{u}}(k)}_2
\end{align*}
\end{proof}

\begin{lem}\label{lem:dist_u}
	\begin{align*}
	\norm{u_{i,v,l_t,t}-\hat{u}_{i,t}}_2 \le \eta \ell Q'(0,t) B \text{ where B is a constant}
	\end{align*}
\end{lem}

\begin{proof}[Proof of Lemma~\ref{lem:dist_u}]
\begin{align*}
\norm{u_{i,v,l_t,t}-\hat{u}_{i,t}}_2 &= \norm{\vect{a}_{v,l_t,t}^\top \vect{x}^{(H)}_{v,l_t,t}-\vect{a}_{t}^\top \vect{x}^{(H)}_{t}}_2 \\
& \le \eta (2 a_{2,0}  3 c_{x,0}  \ell Q'(0,t) (1 + \frac{c_x}{\sqrt{m}}))
\end{align*}
\end{proof}

\begin{lem}\label{lem:i2}
If Condition~\ref{cond:linear_converge_resist} holds for $t'=0,\ldots,t-1$ and $\eta\le c\lambda_0H^2n^{-2} \ell^{-2} S^{-1}$ for some small constant $c$, we have
$\norm{\vect{I}_2(t)}_2 \le  C_{I_2}^{*} \eta^2 \norm{y_i-\hat{u}_{i,t-1}}_2$ where $C_{I_2}^{*}$ is a constant and thus $\norm{\vect{I}_2(t)}_2 \le \frac{1}{8}\eta \lambda_0 \norm{\vect{y}-\vect{\hat{u}}(k)}_2$.
\end{lem}

\begin{proof}[Proof of Leamma~\ref{lem:i2}]
\begin{align*}
I_2^i(t) =& \int_{s=0}^{1} \langle \theta(t+1)-\theta(t), \hat{u}'_i\left(\params(t)\right) -\hat{u}'_i\left(\params(t)-s (\params(t)-\params(t+1)) \right) \rangle ds
\end{align*}
Define for $1 \le h \le H$
\begin{align*}
    \hat{u}'^{(h)}_i\left(\params(t)\right) = \frac{\partial \hat{u}(\params(t)) }{\mathcal{W}_{t}^{(h)}}
\end{align*}
And 
\begin{align*}
\hat{u}'^{(H+1)}_i\left(\params(t)\right) = \frac{\partial \hat{u}(\params(t)) }{\vect{a}_t}
\end{align*}
\begin{align*}
	\abs{I_2^i(t)} \le  \max_{0\le s\le 1} & \sum_{h=1}^{H}  \norm{\mathcal{W}_{t}^{(h)}-\mathcal{W}_{t-1}^{(h)}}_F \norm{  \hat{u}'^{(h)}_i\left(\params(t)\right) -\hat{u}'^{(h)}_i\left(\params(t)-s (\theta(t+1)-\theta(t))\right)}_F\\
	& + \norm{\vect{a}_t-\vect{a}_{t-1}}_2 \norm{\hat{u}'^{(H+1)}_i\left(\params(t)\right) -\hat{u}'^{(H+1)}_i\left(\params(t)-s (\theta(t+1)-\theta(t))\right)}_2.
	\end{align*}
From Lemma \ref{lem:dist_from_wholenetwork} and Lemma \ref{lem:dist_from_init_resnet_sub}, 
\begin{align*}
\norm{\mathcal{W}_{t}^{(h)}-\mathcal{W}_{t-1}^{(h)}}_F & \le \eta \ell \hat{Q}'(t-1) \\
\norm{\vect{a}_t-\vect{a}_{t-1}}_2 & \le \eta \ell \hat{Q}'(t-1)
\end{align*}

Let $\vect{x}_{i,t-1,s}^{(l)}$ be the activation of global network with $\mathcal{W}_{t-1,s}=\mathcal{W}_{t-1}-s(\mathcal{W}_{t-1}-\mathcal{W}_{t})$. We similarly define $\mat{J}_{i,t-1,s}^{(l)}$ and $\vect{a}_{t-1,s}$
\begin{align*}
& \norm{  \hat{u}'^{(h)}_i\left(\params(t)\right) -\hat{u}'^{(h)}_i\left(\params(t)-s (\mathcal{W}_{t-1}-\mathcal{W}_{t})\right)}_F \\
\le & \frac{c_{res}}{H\sqrt{m}} \big{||}\vect{x}_{i,t-1,s}^{(h-1)} \cdot
\left[\vect{a}_{t-1,s}^\top \prod_{l=h+1}^{H}\left(\mat{I}+\frac{c_{res}}{H\sqrt{m}}\mat{J}_{i,t-1,s}^{(l)}\mathcal{W}_{t-1,s}^{(l)}  \right) \mat{J}_{i,t-1,s}^{(h)}  \right] \\
& -\vect{x}_{i,t-1}^{(h-1)} \cdot
\left[\vect{a}_{t-1}^\top \prod_{l=h+1}^{H}\left(\mat{I}+\frac{c_{res}}{H\sqrt{m}}\mat{J}_{i,t-1}^{(l)}\mathcal{W}_{t-1}^{(l)}  \right) \mat{J}_{i,t-1}^{(h)}  \right]\big{||}_F \\
\end{align*}

Through similar calculation in Lemma \ref{lem:dist_from_wholenetwork},
\begin{align*}
\norm{\mathcal{W}_{t-1,s}^{(l)}-\mathcal{W}_{t-1}^{(l)}}_F = & s\norm{(\mathcal{W}_{t-1}^{(l)}-\mathcal{W}_{t}^{(l)})}_F \\
\le &  \norm{(\mathcal{W}_{t-1}^{(l)}-\mathcal{W}_{t}^{(l)})}_F \\
\le & \eta \ell \hat{Q}'(t-1)
\end{align*}
\begin{align*}
\norm{\vect{x}_{i,t-1,s}^{(l)}-\vect{x}_{i,t-1}^{(l)}}_2 & \le \frac{c_{res}}{H\sqrt{m}} \norm{\relu{\mathcal{W}^{(l)}_{t-1,s} \vect{x}^{(l-1)}_{t-1,s}} -\relu{\mathcal{W}^{(l)}_{t-1} \vect{x}^{(l-1)}_{t-1}}}_2+\norm{\vect{x}^{(l-1)}_{t-1,s}-\vect{x}^{(l-1)}_{t-1}}_2 \\
& \le  \frac{c_{res}}{H\sqrt{m}} \norm{\relu{\mathcal{W}^{(l)}_{t-1,s} \vect{x}^{(l-1)}_{t-1,s}} -\relu{\mathcal{W}^{(l)}_{t-1,s} \vect{x}^{(l-1)}_{t-1}}}_2 \\ & +\frac{c_{res}}{H\sqrt{m}} \norm{\relu{\mathcal{W}^{(l)}_{t-1,s} \vect{x}^{(l-1)}_{t-1}} -\relu{\mathcal{W}^{(l)}_{t-1} \vect{x}^{(l-1)}_{t-1}}}_2 + \norm{\vect{x}^{(l-1)}_{t-1,s}-\vect{x}^{(l-1)}_{t-1}}_2 \\
& \le \frac{c_{res}L}{H\sqrt{m}} \norm{\mathcal{W}_{t-1,s}^{(l)}}_F\norm{\vect{x}^{(l-1)}_{t-1,s}-\vect{x}^{(l-1)}_{t-1}}_2 \\
& + \frac{c_{res}L}{H\sqrt{m}} \norm{\mathcal{W}_{t-1,s}^{(l)}-\mathcal{W}_{t-1}^{(l)}}_F \norm{\vect{x}^{(l-1)}_{t-1}}_2 + \norm{\vect{x}^{(l-1)}_{t-1,s}-\vect{x}^{(l-1)}_{t-1}}_2 \\
& \le (1 +  \frac{c_{res}L}{H\sqrt{m}} (c_{w,0}\sqrt{m}+R'\sqrt{m}+\eta \ell \hat{Q}'(t-1))) \norm{\vect{x}^{(l-1)}_{t-1,s}-\vect{x}^{(l-1)}_{t-1}}_2 \\
& + \frac{c_{res}L}{H\sqrt{m}} \eta \ell \hat{Q}'(t-1)(c_x R'+c_{x,0}) \\
\end{align*}
Also
\begin{align*}
\norm{\vect{x}_{i,t-1,s}^{(0)}-\vect{x}_{i,t-1}^{(0)}}_2 & = \frac{c_{res}}{H\sqrt{m}} \norm{\relu{\mathcal{W}^{(1)}_{t-1,s} \vect{x}_i} -\relu{\mathcal{W}^{(0)}_{t-1} \vect{x}_i}}_2 \\
& \le \frac{c_{res}}{H\sqrt{m}} L \norm{\mathcal{W}^{(1)}_{t-1,s}  -\mathcal{W}^{(0)}_{t-1} }_2 \\
& \le \frac{c_{res}}{H\sqrt{m}} L \eta \ell \hat{Q}'(t-1)
\end{align*}
Thus
\begin{align*}
\norm{\vect{x}_{i,t-1,s}^{(l)}-\vect{x}_{i,t-1}^{(l)}}_2 & \le (\frac{c_{res}}{H\sqrt{m}} L \eta \ell \hat{Q}'(t-1)+\frac{\frac{c_{res}L}{H\sqrt{m}} \eta \ell \hat{Q}'(t-1)(c_x R'+c_{x,0})}{\frac{c_{res}L}{H\sqrt{m}} (c_{w,0}\sqrt{m}+R'\sqrt{m})}) e^{\frac{c_{res}L}{H\sqrt{m}} (c_{w,0}\sqrt{m}+R'\sqrt{m}+\eta \ell \hat{Q}'(t-1))} \\
& \le \eta \ell \hat{Q}'(t-1) (\frac{c_{res}}{H\sqrt{m}} L+\frac{ (c_x R'+c_{x,0})}{ (c_{w,0}\sqrt{m}+R'\sqrt{m})})e^{\frac{c_{res}L}{\sqrt{m}} (c_{w,0}\sqrt{m}+R'\sqrt{m}+\eta \ell \hat{Q}'(t-1))} \\
& \triangleq \eta \ell \hat{Q}'(t-1) C^{*}_x
\end{align*}
Similarly, through standard calculation we can get
\begin{align*}
\norm{\vect{a}_{t-1,s}-\vect{a}_{t-1}}_2 \le \eta \ell \hat{Q}'(t-1)
\end{align*}
Lastly,
\begin{align*}
\norm{\mat{J}_{i,t-1,s}^{(l)}-\mat{J}_{i,t-1}^{(l)}}_2 & = \norm{\sigma'(\mathcal{W}^{(l)}_{t-1,s} \vect{x}^{(l-1)}_{t-1,s})-\sigma'(\mathcal{W}^{(l)}_{t-1} \vect{x}^{(l-1)}_{t-1})}_2 \\
& \le \beta \norm{\mathcal{W}^{(l)}_{t-1,s} \vect{x}^{(l-1)}_{t-1,s}-\mathcal{W}^{(l)}_{t-1} \vect{x}^{(l-1)}_{t-1}}_2 \\
& \le \beta (\norm{\mathcal{W}_{t-1,s}^{(l)}}_F\norm{\vect{x}^{(l-1)}_{t-1,s}-\vect{x}^{(l-1)}_{t-1}}_2+\norm{\mathcal{W}_{t-1,s}^{(l)}-\mathcal{W}_{t-1}^{(l)}}_F \norm{\vect{x}^{(l-1)}_{t-1}}_2) \\
& \le \beta (\frac{c_{res}L}{H\sqrt{m}} (c_{w,0}\sqrt{m}+R'\sqrt{m}+\eta \ell \hat{Q}'(t-1))\eta \ell \hat{Q}'(t-1) C^{*}_x + \eta \ell \hat{Q}'(t-1)(c_x R'+c_{x,0})) \\
& = \eta \ell \hat{Q}'(t-1) \beta (\frac{c_{res}L}{H\sqrt{m}} (c_{w,0}\sqrt{m}+R'\sqrt{m}+\eta \ell \hat{Q}'(t-1)) C^{*}_x + (c_x R'+c_{x,0})) \\
& \triangleq \eta \ell \hat{Q}'(t-1) \beta  C^{*}_J
\end{align*}
Thus, according to Lemma G.1 in \cite{du2019gradient}, we have
\begin{align*}
& \norm{  \hat{u}'^{(h)}_i\left(\params(t)\right) -\hat{u}'^{(h)}_i\left(\params(t)-s (\mathcal{W}_{t-1}-\mathcal{W}_{t})\right)}_F \\
\le & \frac{c_{res}}{H\sqrt{m}} \big{||}\vect{x}_{i,t-1,s}^{(h-1)} \cdot
\left[\vect{a}_{t-1,s}^\top \prod_{l=h+1}^{H}\left(\mat{I}+\frac{c_{res}}{H\sqrt{m}}\mat{J}_{i,t-1,s}^{(l)}\mathcal{W}_{t-1,s}^{(l)}  \right) \mat{J}_{i,t-1,s}^{(h)}  \right] \\
& -\vect{x}_{i,t-1}^{(h-1)} \cdot
\left[\vect{a}_{t-1}^\top \prod_{l=h+1}^{H}\left(\mat{I}+\frac{c_{res}}{H\sqrt{m}}\mat{J}_{i,t-1}^{(l)}\mathcal{W}_{t-1}^{(l)}  \right) \mat{J}_{i,t-1}^{(h)}  \right]\big{||} \\
& \le \eta \ell \hat{Q}'(t-1) \frac{c_{res}}{H \sqrt{m}} 2 c_{x,0} 2 a_{2,0} L e^{2Lc_{w,0}}  (\frac{C^{*}_x}{2 c_{x,0}}+\frac{1}{2a_{2,0}}+\frac{c_{res}}{\sqrt{m}}\beta C^{*}_J )
\end{align*}
On the other hand
\begin{align*}
    \norm{\hat{u}'^{(H+1)}_i\left(\params(t)\right) -\hat{u}'^{(H+1)}_i\left(\params(t)-s (\theta(t+1)-\theta(t))\right)}_2 & \le \norm{\vect{x}^{(H)}_{t-1,s}-\vect{x}^{(H)}_{t-1}}_2 \\
    & \le \eta \ell \hat{Q}'(t-1) C^{*}_x
\end{align*}

In the end,
\begin{align*}
\norm{\vect{I}_2(t)}_2 & \le \eta \ell \hat{Q}'(t-1) \eta \ell \hat{Q}'(t-1) (\frac{c_{res}}{\sqrt{m}} 2 c_{x,0} 2 a_{2,0} L e^{2Lc_{w,0}}  (\frac{C^{*}_x}{2 c_{x,0}}+\frac{1}{2a_{2,0}}+\frac{c_{res}}{\sqrt{m}}\beta C^{*}_J )+  C^{*}_x)\\
& \le \eta^2 \ell^2 \hat{Q}'(t-1)^2  (\frac{c_{res}}{\sqrt{m}} 2 c_{x,0} 2 a_{2,0} L e^{2Lc_{w,0}}  (\frac{C^{*}_x}{2 c_{x,0}}+\frac{1}{2a_{2,0}}+\frac{c_{res}}{\sqrt{m}}\beta C^{*}_J )+C^{*}_x)\\
& \le  \eta^2 C_{I_2}^{*} \norm{\vect{y}-\vect{\hat{u}}(t)}_2 \\
& \le \frac{1}{16}\eta \lambda_0 \norm{\vect{y}-\vect{\hat{u}}(t)}_2
\end{align*}
\end{proof}
\begin{lem}\label{lem:quadratic_resnet_whole}
If Condition~\ref{cond:linear_converge_resist} holds for $t'=0,\ldots,t-1$ and $\eta\le c\lambda_0H^2n^{-2} \ell^{-2} S^{-1}$ for some small constant $c$, we have
$\norm{\vect{\hat{u}}(t+1)-\vect{\hat{u}}(t)}_2^2 \le \frac{1}{16}\eta \lambda_0 \norm{\vect{y}-\vect{\hat{u}}(k)}_2^2$.
\end{lem}

\begin{proof}[Proof of Lemma~\ref{lem:quadratic_resnet_whole}]
	\begin{align*}
\norm{\vect{\hat{u}}(t+1)-\vect{\hat{u}}(t)}_2^2 = & \sum_{i=1}^{n}\left(\vect{a}_{t+1}^\top \vect{x}_{i,t+1}^{(H)}-\vect{a}_{t}^\top \vect{x}_{i,t}^{(H)}\right)^2 \\
= & \sum_{i=1}^{n}\left(\left[\vect{a}_{t+1}-\vect{a}_{t}\right]^\top \vect{x}_{i,t+1}^{(H)}+\vect{a}_{t}^\top \left[\vect{x}_{i,t+1}^{(H)}-\vect{x}_{i,t}^{(H)}\right] \right)^2 \\
\le &2\norm{\vect{a}_{t+1}-\vect{a}_{t}}_2^2\sum_{i=1}^{n}\norm{\vect{x}_{i,t+1}^{(H)}}_2^2+2\norm{\vect{a}_{t}}_2^2\sum_{i=1}^{n}\norm{\vect{x}_{i,t+1}^{(H)}-\vect{x}_{i,t}^{(H)}}_2^2\\
\le &18n\eta^2 \ell^2 c_{x,0}^2Q'(t)^2+4 n \left(\eta  \ell a_{2,0}c_xQ'(t)\right)^2\\
\le &\frac{1}{8}\eta \lambda_0 \norm{\vect{y}-\vect{\hat{u}}(t)}_2^2.
\end{align*}
\end{proof}

\end{document}